\documentclass[letterpaper]{article} 
\usepackage{aaai24}  
\nocopyright
\usepackage{times}  
\usepackage{helvet}  
\usepackage{courier}  
\usepackage[hyphens]{url}  
\usepackage{graphicx} 
\usepackage{natbib}  
\usepackage{caption} 
\usepackage{amsmath,amssymb}
\usepackage{amsthm}
\usepackage{xcolor}
\usepackage{newfloat}
\usepackage{listings}
\usepackage{diagbox}
\usepackage{algorithm}
\usepackage{algpseudocode}
\usepackage{amsthm}
\usepackage[english]{babel}
\usepackage{hhline}
\usepackage{subcaption}
\usepackage{enumitem}
\usepackage{graphicx}
\newtheorem{theorem}{Theorem}
\newtheorem{lemma}[theorem]{Lemma}
\theoremstyle{remark}
\newtheorem{remark}[theorem]{Remark}

\usepackage{microtype}
\usepackage{graphicx}
\usepackage{booktabs} 
\usepackage{amssymb}
\usepackage{multirow}
\usepackage[textsize=tiny, textwidth=1.1cm]{todonotes}

\usepackage{tikz}
\usepackage{times}
\usepackage{pbox}
\usepackage{mathtools} 
\usepackage{enumitem}
\usepackage{amsxtra}

\DeclareUnicodeCharacter{2212}{-}

\usepackage{newfloat}
\usepackage{listings}
\DeclareCaptionStyle{ruled}{labelfont=normalfont,labelsep=colon,strut=off} 
\floatstyle{ruled}
\newfloat{listing}{tb}{lst}{}
\floatname{listing}{Listing}
\title{RGMComm: Return Gap Minimization via Discrete Communications in Multi-Agent Reinforcement Learning} 
\author {
    Jingdi Chen\textsuperscript{\rm 1},
    Tian Lan\textsuperscript{\rm 1},
    Carlee Joe-Wong\textsuperscript{\rm 2}
}
\affiliations {
    \textsuperscript{\rm 1}The George Washington University\\
    \textsuperscript{\rm 2}Carnegie Mellon University\\
    jingdic@gwu.edu, tlan@gwu.edu, cjoewong@andrew.cmu.edu
}

\begin{document}

\maketitle


\begin{abstract}

Communication is crucial for solving cooperative Multi-Agent Reinforcement Learning tasks in partially observable Markov Decision Processes. Existing works often rely on black-box methods to encode local information/features into messages shared with other agents, leading to the generation of continuous messages with high communication overhead and poor interpretability. Prior attempts at discrete communication methods generate one-hot vectors trained as part of agents' actions and use the Gumbel softmax operation for calculating message gradients, which are all heuristic designs that do not provide any quantitative guarantees on the expected return. 
This paper establishes an upper bound on the return gap between an ideal policy with full observability and an optimal partially observable policy with discrete communication. This result enables us to recast multi-agent communication into a novel online clustering problem over the local observations at each agent, with messages as cluster labels and the upper bound on the return gap as clustering loss. To minimize the return gap, we propose the Return-Gap-Minimization Communication (RGMComm) algorithm, which is a surprisingly simple design of discrete message generation functions and is integrated with reinforcement learning through the utilization of a novel Regularized Information Maximization loss function, which incorporates cosine-distance as the clustering metric. Evaluations show that RGMComm significantly outperforms state-of-the-art multi-agent communication baselines and can achieve nearly optimal returns with few-bit messages that are naturally interpretable. 
\footnote{Codes are available at: https://github.com/JingdiC/RGMComm.git. Appendix is in~\cite{chen2023rgmcomm}.}

\end{abstract}

\section{Introduction}
In multi-agent tasks, communication is necessary to successfully complete tasks when agents have partial observability of the environment. Multi-agent reinforcement learning (MARL) has recently seen success in scenarios that require communication~\cite{10.5555/646288.686470,6303906}.
Existing approaches to multi-agent communications have considered \textit{continuous communication}~\cite{dial,commnet,maddpg,atoc,wang2019learning,sarnet} and \textit{discrete communication}~\cite{emergentlan,freed2020sparse,lazaridou2020emergent,9812285,tucker2022trading}.
However, \textit{continuous communication messages} refer to numerical vectors originating from a continuous space. These vectors are generated by encoding local information/features using Deep Neural Networks (DNNs) or attention networks. This is largely a black-box approach with high communication overhead and offers little explainability of the messages. 
Efforts on generating \textit{discrete communication messages} enable agents to broadcast one-hot vectors apt for solving particular tasks. 
However, discretizing messages through the imposition of one-hot vectors, as outlined in~\cite{maddpg, commnet, freed2020sparse, lazaridou2020emergent, 9812285}, inherently precludes agents from learning some desirable properties of the messages. This is because the one-hot vectors do not establish relationships between messages, since each one-hot vector remains orthogonal to and equally distant from all other one-hot vectors.
There are some recent works that try to address the limitations of one-hot communication by endowing agents with learnable messages~\cite{tucker2021emergent} or discretized bottleneck layer outputs of autoencoders and using them as communication vectors ~\cite{tucker2022trading}, but none of these methods provides average return guarantees for a Decentralized Partially Observable Markov Decision Process (Dec-POMDP) with communication, and the communication module in ~\cite{tucker2022trading} still requires a large vocabulary size. 


In this paper, we propose a discrete MARL communication algorithm, Return-Gap-Minimization Communication (RGMComm), that is proven to achieve a closed-form guarantee on the optimal expected average return for Dec-POMDPs. It analyzes what information in an agent's local observations is relevant for other agents' decision-making (for minimizing the resulting return gap) and supports sharing messages between MARL agents using any finite-size discrete alphabet. 
RGMComm determines the message each agent sends to others, given its current observed local state. 

Our key insight is that the problem of generating messages from local observations can be viewed as clustering the local observations into clusters corresponding to different message labels.
More precisely, consider a Dec-POMDP problem with $n$ agents. Each agent $j$ computes a message $m_{j}$ sent to each other agent $i$ based on its local observation $o_j$ in each step $t$. Let $\boldsymbol{m_{-i}}=\{m_{j},\forall j \neq i\}$ be the collection of all messages received by agent $i$. 
Then the policy of agent $i$ with communication is conditioned on local observation $o_i$ and messages $\boldsymbol{m_{-i}}$ rather than the full observation $o_1,\ldots,o_n$. Its impact on the optimal expected average return can be characterized using a Policy Change Lemma, relying on the average distance between joint action-values in each cluster. Intuitively, if the action-values corresponding to the observations in each cluster are likely maximized at the same actions, the policy with communication would have a small return gap with the policy assuming full observations. 
Since minimizing the return gap directly is intractable, we derive a closed-form upper bound and minimize it to find the optimal communication strategy, which is a common technique in Machine Learning (ML)~\cite{duan2017fast,hoffman2016elbo}.
For any discrete communication strategy, we quantify the gap between the optimal expected average return of an ideal policy with full observability, i.e., $\pi^*=[\pi_i^*(a_i|o_1,\ldots,o_n),\forall i]$ and the optimal expected average return of a communication-enabled, partially-observable policy, i.e., $\pi=[\pi_i(a_i|o_i,\boldsymbol{m_{-i}}),\forall i]$. 
This return gap is proven to be bounded by $O(\sqrt{\epsilon}nQ_{\rm max})$, with respect to the number of agents $n$, the highest action-value $Q_{\rm max}$ and the average cosine-distance $\epsilon$ between joint action-value vectors corresponding to the same messages.

To the best of our knowledge, ours is the first theoretical result quantifying the value of communication in Dec-POMDP through a closed-form return bound. We proposed RGMComm, which trains an online clustering network via a Regularized Information Maximization (RIM) loss function~\cite{imsat}, consisting of \textbf{a novel cosine-distance clustering loss} $I_{\rm CD}$ specifically designed to minimize the proven return bound for message label generation, and a mutual information regularization loss $I_{\rm MI}(o_j,m_{j})$ to guarantee even cluster sizes and distinct observations-to-cluster assignments.
The main contributions of the paper are as follows:
\begin{itemize}
\item We quantify the return gap between an ideal policy with full observability and a partially-observable policy with communication via a closed-form upper-bound.
\item We introduce RGMComm, a discrete communication framework for MARL, aimed at minimizing the upper bound of the return gap. Instead of generating messages directly via DNNs, RGMComm employs a finite-size discrete alphabet and treats message generation as an online clustering problem.
\item With few-bit messages, RGMComm significantly outperforms state-of-the-art communication-enabled MARL algorithms with partial observability, including CommNet~\cite{commnet}, MADDPG~\cite{maddpg}, IC3Net~\cite{ic3net}, TarMAC~\cite{tarmac}, SARNet~\cite{sarnet} and VQ-VIB~\cite{tucker2022trading} and achieves nearly-optimal returns.
\end{itemize}

\section{Related Work}

\label{sec:related_work_main}
Previous work in MARL communication mostly establishes how agents should learn to communicate assuming continuous communication vectors~\cite{dial,ic3net,atoc, tarmac,sarnet}. Inspired by the efficient communication patterns observed in human interactions, where people only use a discrete set of words, some previous works enable agents to learn sparse and discrete communication. 
CommNet~\cite{commnet} and MADDPG~\cite{maddpg} learn continuous communication vectors alongside their policy and can generate discrete 1-hot binary communication vectors using a pre-defined set of communication alphabets. However, CommNet uses a large single network for all agents, so it cannot easily scale and would perform poorly in an environment with a large number of agents. MADDPG adapts the actor-critic framework using a centralized critic that takes as input the observations and actions of all agents and trains an independent policy network for each agent to generate communication as part of the actions, where each agent would learn a policy specializing in specific tasks. The policy network easily overfits the number of agents, which is infeasible in large-scale MARL. 

Further, communication in terms of sequences of discrete symbols are investigated in~\cite{havrylov2017emergence,emergentlan}. Both of these works generate categorical messages 
and adapt the Gumbel-Softmax Estimator~\cite{jang2016categorical,maddison2016concrete}  to make the communication models with discrete labels differentiable, training them with backpropagation algorithms. However, \cite{havrylov2017emergence} only studies a two-agent setting, and \cite{emergentlan} only learns message meanings through the prompt rewards in a small action space. Many works, therefore, resort to differentiable discrete communication such as~\cite{freed2020sparse,lazaridou2020emergent,9812285,tucker2022trading} where agents are allowed to directly optimize each other’s communication policies through gradients. However, these approaches impose a strong constraint on the nature of communication, which limits their applicability to many real-world multi-agent coordination tasks. 
Besides, the primary limitation is that these approaches lack rigorous policy regret minimization guarantees, operating as heuristic designs.

\section{Preliminaries and Problem Formulation}
\label{sec:4}

\noindent {\bf Dec-POMDP: } A Dec-POMDP~\cite{bernstein2002complexity} models cooperative MARL, where agents lack complete information about the environment and have only local observations. 
We formulate a Dec-POMDP with communication as a tuple $D=\langle S, A, P, \Omega, O, I, n, R, \gamma, g \rangle$, where $S$ is the joint \textbf{state} space and $A=A_1\times A_2 \times \dots \times A_n$ is the joint \textbf{action} space, where $\boldsymbol{a}=(a_1,a_2,\dots,a_n)\in A$ denotes the joint action of all agents.
$P(\boldsymbol{s}'|\boldsymbol{s},\boldsymbol{a}): S \times A \times S \to [0,1] $ is the \textbf{state transition function}. 
$\Omega$ is the \textbf{observation} space. $O(\boldsymbol{s}, i): S \times I \to \Omega$ is a function that maps from the joint state space to distributions of observations for each agent $i$. 
$I = \{1,2,\dots,n\}$ is a set of $n$ agents, $R(\boldsymbol{s}, \boldsymbol{a}): S \times A \to \mathbb{R}$ is the \textbf{reward function} in terms of state $\textbf{s}$ and joint action $\boldsymbol{a}$, $\gamma$ is the discount factor, and 
$g: \Omega \to M$ is the \textbf{message generation function} that each agent $j$ uses to encode its local observation $o_j$ into a communication message for other agents $i \neq j$. 

\noindent {\bf Return Gap Minimization: }Given a policy $\pi$, we consider the average expected return $J(\pi) =\lim_{T \to \infty} (1/T) E_{\pi} [{\sum_{t=0}^T R_{t}}]$. The goal of this paper is to minimize the return gap between an ideal policy $\pi^*=[\pi_i^*(a_i|o_1,\ldots,o_n),\forall i]$ with full observability and a partially-observable policy with communications $\pi=[\pi_i(a_i|o_i,\boldsymbol{m_{-i}}),\forall i]$ where message labels $\boldsymbol{m_{-i}}=\{m_{j}=g(o_j), \forall j \neq i\}$, i.e.,
\begin{equation}
    \min J(\pi^*) - J(\pi,g).
\end{equation}
While the problem is equivalent to maximizing $J(\pi,g)$, the return gap can be analyzed more easily by contrasting $\pi$ and $\pi^*$. We derive an upper bound of the return gap and then design efficient communication strategies to minimize it. 
We consider the discounted observation-based state value and the corresponding action-value functions for the Dec-POMDP: 
\begin{equation} \label{eq:v}
\setlength{\abovedisplayskip}{1pt}
\setlength{\belowdisplayskip}{1pt}
\begin{aligned}
    V^{\pi}(\boldsymbol{o}) = \mathbb{E}_{\pi}[\sum_{i=0}^{\infty}\gamma^i \cdot R_{t+i} \Big|\boldsymbol{o}_t=\boldsymbol{o}, \boldsymbol{a}_{t} \sim \pi], \\
Q^{\pi}(\boldsymbol{o},\boldsymbol{a})=\mathbb{E}_{\pi}[\sum_{i=0}^{\infty}\gamma^i \cdot R_{t+i} \Big|\boldsymbol{o}_t=\boldsymbol{o}, \boldsymbol{a}_{t}= \boldsymbol{a}],
\end{aligned}
\end{equation}
where $t$ is the current time step. Re-writing the average expected return as an expectation in terms of $V^{\pi}(\boldsymbol{o})$:
\begin{equation}\label{eq:J_v}
J(\pi) = \lim_{\gamma\rightarrow 1} E_{\mu}[(1-\gamma)V^{\pi}(\boldsymbol{o})],
\end{equation}
where $\mu$ is the initial observation distribution at time step $t=0$, i.e., $\boldsymbol{o}(0) \sim \mu$. We will leverage this state-value function
$V^{\pi}(\boldsymbol{o})$ and its corresponding action-value function $Q^{\pi}(\boldsymbol{o},\boldsymbol{a})$ to unroll the Dec-POMDP and derive a closed-form upper-bound to quantify the return gap. For $V^{\pi}(\boldsymbol{o})$ and $Q^{\pi}(\boldsymbol{o},\boldsymbol{a})$ we suppress $g$ for simpler notations.


\section{Upper-Bounding the Return Gap}
\label{sec:theory}
In this section, we present the theoretical statements and proof sketches of our results, followed by an illustrative example. \textbf{Appendix A\cite{chen2023rgmcomm} contains a notation table, a detailed proof outline, and complete proofs for all theoretical results}. Our assumptions are that the observation and action spaces in the Dec-POMDP tuple (explained in Sec.~\ref{sec:4}) are discrete with finite observations and actions, which is a technical condition to simplify the proof. For continuous observation/action spaces, we can extend the results by considering the cosine-distance between action-value functions and replacing summations with integrals, or sampling the action-value functions as an approximation. We use $V^*$  for $V^{\pi^*}$ and $Q^*$ for $Q^{\pi^*}$ for simplicity. 

\begin{lemma}\label{lemma:1}
{(Policy Change Lemma.)} For any policies $\pi^*$ and $\pi$, the optimal expected average return gap is bounded by:
\begin{equation} \label{equ:regret_1_main}
\setlength{\belowdisplayskip}{1pt}
    \begin{aligned}
    J(\pi^*)-J(\pi) &\le \sum_{m} \sum_{\boldsymbol{o}\sim m} [Q^{*}(\boldsymbol{o},\boldsymbol{a}_t^{\pi^*}) - Q^{\pi}(\boldsymbol{o},\boldsymbol{a}_t^{\pi})] d_{\mu}^{\pi}(\boldsymbol{o}),  \\
   d^{\pi}_{\mu}(\boldsymbol{o})&=(1-\gamma)\sum_{t=0}^{\infty} \gamma^{t} \cdot P(\boldsymbol{o}_t=\boldsymbol{o}|\pi,\mu),
    \end{aligned}
\end{equation}
where $d^{\pi}_{\mu}(\boldsymbol{o})$ is the $\gamma$-discounted visitation probability under policy $\pi$ and initial observation distribution $\mu$, and $\sum_{\boldsymbol{o}\sim m}$ is a sum over all observations corresponding to message $m$.
\end{lemma}
\textbf{Proof Sketch.} Our key idea is to leverage the state value function $V^{\pi}(\boldsymbol{o})$ and its corresponding action value function $Q^{\pi}(\boldsymbol{o},\boldsymbol{a})$ in Eq.(\ref{eq:v}) to unroll the Dec-POMDP from timestep $t=0$ and onward. The detailed proof is provided in Appendix A.

Then we define the action value vector corresponding to observation $o_j$, i.e.,
\begin{equation} \label{equ:vector_re_write_Q_1}
    \setlength{\abovedisplayskip}{1pt}
    \setlength{\belowdisplayskip}{1pt}
    \begin{aligned}
    \bar{Q}^*(o_j)=[\widetilde{Q}^*(o_{-j},o_j), \forall{o_{-j}}],
    \end{aligned}
\end{equation}
where $o_{-j}$ are the observations of all other agents and $\widetilde{Q}^*(o_{-j},o_j)$ is a vector of action-values weighted by marginalized visitation probabilities $d_{\mu}^{\pi}(o_i|o_j)$ and corresponding to different actions, i.e., $\widetilde{Q}^*(o_{-j},o_j)=[Q^*(o_{-j},o_j,\boldsymbol{a})\cdot d_{\mu}^{\pi}(o_{-j}|o_j), \forall \boldsymbol{a}]$.

Next, we view the message generation $m_{j}=g(o_j)$ as a clustering problem with $o_j$ as input and $m_{j}$ as labels, i.e., the $o_j$ are divided into clusters and each is labeled with the clustering label to be sent out to all other agents $i \neq j$.
We bound the policy gap between $\pi^*_{(j)}$ and $\pi_{(j)}$, which are optimal policies conditioned on $o_j$ and $m_{j}$, using the average cosine-distance of action value vectors $\bar{Q}^*(o_j)$ corresponding to $o_j$ in the same cluster and its cluster center $\bar{H}(m)=\sum_{o_j\sim m}\bar{d}_m(o_j)\cdot\bar{Q}^*(o_j)$ under each message $m$. Here $\bar{d}_m(o_j)=d_{\mu}^{\pi}(o_j)/d_{\mu}^{\pi}(m)$ is the marginalized probability of $o_j$ in cluster $m$ and $d_{\mu}^{\pi}(m)$ is the probability of message $m$ under policy $\pi$, and the environments' initial observation distribution is represented by $\boldsymbol{o}(t=0) \sim \mu$.
To this end, we let $\epsilon(o_j)= D_{cos}( \bar{Q}^*(o_j), \bar{H}(m))$ be the cosine-distance between vectors $\bar{Q}^*(o_j)$ and $\bar{H}(m)$ and consider the \textbf{average cosine-distance} $\epsilon$ across all clusters represented by different message labels $m$, which is defined as:
\begin{equation} \label{equ:epsilon_main}
    \epsilon \triangleq \sum_m d_{\mu}^{\pi}(m) \sum_{o_j\sim m} \bar{d}_m(o_j) \cdot \epsilon(o_j), 
\end{equation}
The result is summarized in Lemma~\ref{lemma:regret_0}.
\begin{lemma}   
\label{lemma:regret_0}
(Impact of Communication.) Consider two optimal policies $\pi^*_{(j)}$ and $\pi_{(j)}$ conditioned on $o_j$ and $m_{j}$, respectively, while the observability and policies of all other agents remain the same. The optimal expected average return gap is bounded by:
\begin{equation} \label{equ:regret_3}
    \begin{aligned}
    & J({\pi}^*_{(j)}) - J({\pi_{(j)}},g) \le  O(\sqrt{\epsilon}Q_{\rm max})
    \end{aligned}
\end{equation}
where $Q_{\rm max}$ is the maximum absolute action value of $\bar{Q}^*(o_j)$ in each cluster as $Q_{\rm max}=max_{o_j}||\bar{Q}^*(o_j)||_2$, and $\epsilon$ is the average cosine-distance defined in Eq.(\ref{equ:epsilon_main}).
\end{lemma}
\textbf{Proof Sketch.} We give an outline and provide the complete proof in Appendix A.
Since the observability of all other agents $i\neq j$ remains the same, we consider them as a conceptual agent denoted by $-j$. 
For simplicity, we use $\pi^*$ to represent ${\pi}^*_{(j)}$, and $\pi$ to represent ${\pi}_{(j)}$ in distribution functions.



\textit{Step 1: Recasting communication into online clustering}. In this view, policy ${\pi_{(j)}}$ (conditioned on $m_{j}$) is restricted to taking the same actions for all $o_j$ in the same cluster and under the same message $m_{j}$ (the clustering label). 

\textit{Step 2: Rewrite the return gap in vector form}.
We define an auxiliary function $\Phi_{\rm max}(\boldsymbol{X})$ that returns the largest component of vector $\boldsymbol{X}$. 
Since the optimal average return $J(\pi^*_{(j)})$ is conditioned on complete $o_j$, it can achieve the maximum for each vector $\bar{Q}^*(o_j)$. Thus, $\sum_{o_j \sim m}\bar{d}_m(o_j) \cdot \Phi_{\rm max}(\bar{Q}^*(o_j))$ can be defined as selecting the action from optimal policy $\pi^*_{(j)}$ where each agent chooses a different action distribution to maximize $(\bar{Q}^*(o_j))$.  
On the other hand, policy $\pi_{(j)}$ is conditioned on messages $m_{j}$ rather than complete $o_j$ and thus must take the same actions for all $o_j$ in the same cluster. Hence we can construct a (potentially sub-optimal) policy to achieve $\Phi_{\rm max}(\sum_{o_j\sim m}\bar{d}_m(o_j)\cdot\bar{Q}^*(o_j))$ which provides a lower bound on $J(\pi_{(j)}, g)$. (The Appendix details the transformation to vector terms with the help of $\Phi_{\rm max}(\boldsymbol{X})$.)

Using the transformed vector term formats of optimal returns $J(\pi^*_{(j)})$ and $J(\pi_{(j)},g)$ ($J(\pi_{(j)},g)$ is lower bounded), 
we can obtain an upper bound on the return gap:
\begin{equation} \label{equ:regret_2_main}
    \setlength{\abovedisplayskip}{0.8pt}
    \setlength{\belowdisplayskip}{1pt}
    \begin{aligned}
    &J(\pi^*_{(j)})-J(\pi_{(j)})\\
    &\le \sum_{m}d_{\mu}^{\pi}(m)[ \sum_{o_j\sim m}\bar{d}_m(o_j) \cdot \Phi_{\rm max}(\bar{Q}^*(o_j)) \\
    &- \Phi_{\rm max}(\sum_{o_j\sim m}\bar{d}_m(o_j)\cdot\bar{Q}^*(o_j)) ],
    \end{aligned}
\end{equation}

  \textit{Step 3: Projecting action-value vectors toward cluster centers.} 
The policy $\pi_{(j)}$ conditioned on $m_{j}$ takes actions based on the action-value vector at the cluster center $\bar{H}(m)$. For each pair of two vectors $\bar{Q}^*(o_j)$ and $\bar{H}(m)$ with $D(\bar{Q}^*(o_j),\bar{H}(m)) \le \epsilon(o_j)$, we use $\cos{\theta_{o_j}}$ to denote the cosine-similarity between each $\bar{Q}^*(o_j)$ and its center $\bar{H}(m)$. Then we have the cosine distance $D(\bar{Q}^*(o_j),\bar{H}(m)) =1- \cos{\theta_{o_j}} \le \epsilon(o_j)$.
By projecting $\bar{Q}^*(o_j)$ toward $\bar{H}(m)$, $\bar{Q}^*(o_j)$ could be re-written as $\bar{Q}^*(o_j) = Q^{\perp}(o_j) + \cos{\theta_{o_j}} \cdot \bar{H}_m$,
then we could upper bound $\Phi_{max}(\bar{Q}^*(o_j))$ by:
\begin{equation*}
    \setlength{\abovedisplayskip}{1pt}
    \setlength{\belowdisplayskip}{1pt}
    \Phi_{\rm max}(\bar{Q}^*(o_j))\le \Phi_{\rm max}(\cos{\theta}_{o_j}\cdot\bar{H}_m) + \Phi_{\rm max}(Q^{\perp}(o_j)). 
\end{equation*}
Taking a sum over all $o_j$ in the cluster, we have $\sum_{o_j\sim m}\bar{d}_m(o_j)\Phi_{\rm max}(\cos{\theta}_{o_j}\cdot\bar{H}_m)=\Phi_{\rm max}(\bar{H}_m)$, since the projected components $\cos{\theta_{o_j}}\cdot\bar{H}_m$ should add up to exactly $\bar{H}_m$. To bound Eq.(\ref{equ:regret_2_main})'s return gap, it remains to bound the orthogonal components $Q^{\perp}(o_j)$.

  \textit{Step 4: Deriving the upper bound w.r.t. cosine-distance. }
Let $||\cdot||_2$ be the $L_2$ norm, then the maximum function $\Phi_{\rm max}(Q^{\perp}(o_j))$ can be bounded by its $L_2$ norm $C \cdot ||Q^{\perp}(o_j)||_2$ for some constant $C$, i.e., $\Phi_{max}(Q^{\perp}(o_j)) \le C\cdot||Q^{\perp}(o_j)||_2$.
Since $Q^{\perp}(o_j)=\bar{Q}^*(o_j)\cdot sin(\theta)$, $\Phi_{\rm max}(Q^{\perp}(o_j))$ can be further upper bounded by $C\cdot ||\bar{Q}^*(o_j)||_2 \cdot |\sin{\theta}|$.
Denote a constant $Q_{\rm max}$ as the maximum absolute action value of $\bar{Q}^*(o_j)$ in each cluster as $Q_{\rm max}=max_{o_j}||\bar{Q}^*(o_j)||_2$; combined with $|sin(\theta)| = \sqrt{1-cos^2(\theta)}=\sqrt{1-[1-\epsilon(o_j)]^2}$, we can thus obtain 
\begin{equation}\label{eq:q_max_main}
\setlength{\abovedisplayskip}{1pt}
\setlength{\belowdisplayskip}{1pt}
    \begin{aligned}
    &\Phi_{max}(Q^{\perp}(o_j)) \le O(\sqrt{\epsilon(o_j)}Q_{\rm max}).
    \end{aligned}
\end{equation}
Using the concavity of the square root with Eq.(\ref{equ:epsilon_main}), i.e., $\sum_{m}d_{\mu}^{\pi}(m)[ \sum_{o_j\sim m}\bar{d}_m(o_j) \cdot \sqrt{\epsilon(o_j)}]\le \sqrt{\epsilon}$, we derive the desired upper bound $J(\pi^*_{(j)})-J(\pi_{(j)}) \le  O(\sqrt{\epsilon}Q_{\rm max})$. 


\begin{theorem} \label{thm:thm_1}
In $n$-agent Dec-POMDP, the return gap between policy $\pi^*$ with full-observability and communication-enabled policy $\pi$ with partial-observability is bounded by:
\begin{equation} \label{equ:regret_4}
    \begin{aligned}
    & J({\pi}^*) - J({\pi},g) \le  O\left(\sqrt{\epsilon} n Q_{\rm max}\right).
    \end{aligned}
\end{equation}
\end{theorem}
Beginning from $\pi^*=[\pi_i^*{(a_i|o_1,o_2,\dots,o_n),\forall i}]$, we can construct a sequence of $n$ policies, each replacing the conditioning on $o_j$ by messages $m_{j}$, for $j=1$ to $j=n$, one at a time. This will result in policy $\pi$ with partial observability. Applying Lemma~\ref{lemma:regret_0} for $n$ times, we prove the upper bound between $J({\pi}^*)$ and $J({\pi},g)$ in this theorem.


\begin{remark} \label{remark:label_M}
    Thm.~\ref{thm:thm_1} holds for any arbitrary finite number of message labels $|M|$. Furthermore, increasing $|M|$ reduces the average cosine distance (since more clusters are formed) and, consequently, a reduction in the return gap due to the upper bound derived in Thm.~\ref{thm:thm_1}.
\end{remark}

\begin{figure}[htbp]
\centering
\includegraphics[width=3.0in, height=1.089in]{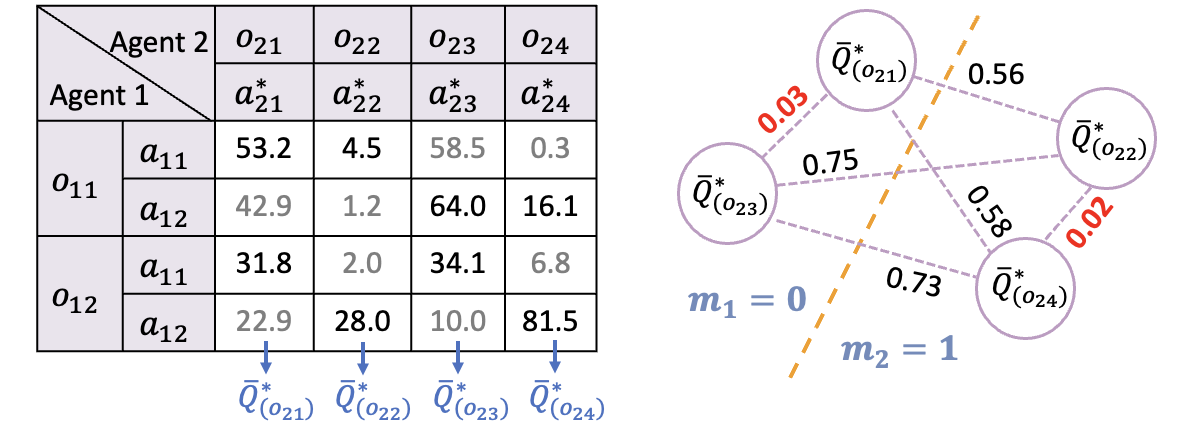}
\vspace{-0.1in}
\caption{An illustrative example of optimizing message generation $m_{2}$ via clustering to minimize the return gap.}
\label{fig:toy_example} 
\vspace{-0.1in}
\end{figure}

\noindent {\bf An illustrative example.} We consider a two-agent matrix game with pay-off $Q^*$ shown in Fig~\ref{fig:toy_example}. Assume that different observations are equally likely and that agent $2$ has a single action in each state. It is easy to see that under an ideal optimal policy $\pi^*=[\pi_1^*,\pi_2^*]$ with full observability $(o_1,o_2)$ can achieve an optimal average reward of $J(\pi_1^*)=1/2\cdot [(53.2+4.5+64.0+16.1)/4 +(31.8+28.0+34.1+81.5)/4]=39.1375$ by choosing the optimal actions in each $(o_1,o_2)$, i.e., $\pi^*={\rm argmax}_{a_1,a_2}Q^*(o_1,o_2,a_1,a_2)$. 
Consider a POMDP scenario where agent $1$ has only local observation $o_1$, and agent 2 can send a 1-bit message $m_{2}=g(o_2)$ encoding its local observation $o_2$. Since the message is only 1-bit but there are 4 possible observations $o_2$, multiple observations must be encoded into the same message. Thus, agent 1 is restricted to a limited policy class $\pi_1(a_1|o_1,m_{2})$ taking the same actions (or action distributions) under the same $o_1$ and for all $o_2$ corresponding to the same message $m_{2}$. Thm.~\ref{thm:thm_1} shows that to minimize the return gap between $\pi_1(a_1|o_1,m_{2})$ with partial observation and the ideal optimal policy $\pi^*=[\pi_1^*,\pi_2^*]$ with full observability, we can leverage a message generation function to minimize the average cosine distance $\epsilon$ between action values for different observations $o_2$. In particular, let $\bar{Q}^*(o_{2})$ be a column of $Q^*$ corresponding to $o_{2}$. There are then four possibilities of $o_2$: $o_{21},o_{22},o_{23},o_{24}$. If $\bar{Q}^*(o_{2i})$ and $\bar{Q}^*(o_{2j})$ are likely maximized at the same actions $a_1$, then encoding $o_{2i}$ and $o_{2j}$ to the same message would result in little return gap, as formalized in Thm.~\ref{thm:thm_1}. The message generation could be viewed as a clustering problem over $o_2$ under vector cosine-distance~\cite{muflikhah2009document}. The right figure in Fig~\ref{fig:toy_example} displays the cosine distances calculated between every pair of action value vectors in $\bar{Q}^*(o_{2i})$ and $\bar{Q}^*(o_{2j})$ (which are readily available during centralized training). According to Thm.~\ref{thm:thm_1}, assigning the same message labels to observations $o_2$ whose optimal action value vectors have smaller cosine distances would result in smaller return gaps. Hence, we assign message label `0' to $\bar{Q}^*(o_{21})$ and $\bar{Q}^*(o_{23})$, whose cosine distance is $0.03$, and message label `1' to $\bar{Q}^*(o_{22})$ and $\bar{Q}^*(o_{24})$, whose cosine distance is $0.02$. Accordingly, the message generation in Fig~\ref{fig:toy_example} leads to an average return of $J(\pi,g) = 38.775$ with an optimal return gap of $J(\pi^*)-J(\pi,g)=0.3625$.

\section{RGMComm Design: Minimizing the Upper Bound of Average Cosine Distance $\epsilon$}

Theorem \ref{thm:thm_1} inspires the RGMComm algorithm (shown in Fig. \ref{fig:algorithm}), which minimizes the upper-bounded return gap between fully-observable and communication-enabled policies. RGMComm employs the online clustering network $g$ to generate message labels, minimizing $\epsilon$ with a cosine-distance loss. Implementation is simple since (1) many MARL algorithms~\cite{maddpg,rashid2018qmix,wang2020dop,zhang2021fop} adopt the actor-critic framework, granting access to action-value functions via DNNs for message training; and (2) Training communication (using online clustering) alongside agent policies becomes feasible by efficiently sampling action-values from replay buffers, making our approach adaptable to continuous state-space problems.


\begin{figure}[htbp]
\centering
\includegraphics[width=2.1in, height=1.53in]{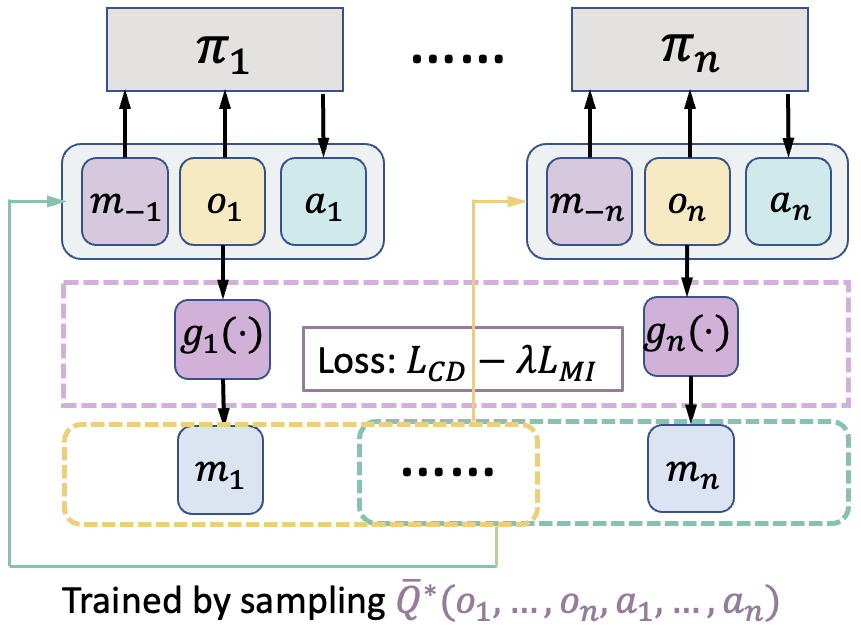}
\vspace{-0.1in}
\caption{RGMComm trains message generation function $g$ via sampled action-value vectors, shaping the cosine-distance loss function $L_{CD}$.}
\label{fig:algorithm} 
\vspace{-0.1in}
\end{figure}

 \noindent {\bf MARL training: } 
Using the actor-critic framework~\cite{maddpg}, we train agent policies and estimate the action-value function. Let $\hat{Q}_{\omega}(\boldsymbol{o},\boldsymbol{a})$ be a DNN with parameters $\omega$ (we drop time $t$ to simplify notations). We update $\omega$ through an experience replay buffer $\mathcal{R}$ containing tuples of agent experiences. The action-value function is updated by minimizing the loss function $\mathcal{L}(\omega) = E_{(\boldsymbol{o}, \boldsymbol{a}, R, \boldsymbol{o}')}[(\hat{Q}_{\omega}(\boldsymbol{o},\boldsymbol{a})-y)^2]$, where $y=R + \gamma \hat{Q}_{\omega'}(\boldsymbol{o}',\boldsymbol{a}')|_{\boldsymbol{a}'}=\pi'(\boldsymbol{o}')$.
Decisions are based on decentralized policies conditioned on partial observation and communication messages, denoted as $\pi=[\pi_i(a_i|o_i,\boldsymbol{m_{-i}})|_{\boldsymbol{m_{-i}}=\{g(o_j),\forall j \neq i\}},\forall i]$. Agents' policies are parameterized by $\{\theta_1,\dots,\theta_n\}$, and these policies are updated and executed in a decentralized manner using the policy gradient to minimize the expected return $J(\theta_i)=E[R_i]$ of each agent $i$: $\nabla_{\theta_i}J(\theta_i) =  E_{\boldsymbol{o},\boldsymbol{a} \sim \mathcal{R}}\left[\nabla_{\theta_i} \log \pi_{\theta_i}(a_i|o_i,\boldsymbol{m_{-i}}) \hat{Q}_{\omega}^{\pi_{\theta_i}} (\boldsymbol{o},\boldsymbol{a})\right]$.

\noindent {\bf Learning message generation: }
The message generation functions $g=\{g_1,\dots,g_n\}$ of all agents are approximated using DNNs parameterized by $\xi=\{\xi_1,\dots,\xi_n\}$. For each agent $j$, we first sample a random minibatch of $K_1$ samples $\mathcal{X}_{j}=(\boldsymbol{o}^{k_1},\boldsymbol{a}^{k_1},R^{k_1},\boldsymbol{o}'^{k_1})$ from the transitions recorded in replay buffer $\mathcal{R}$, which contains the observation-action pairs from all agents including agent $j$. Then we sample a set $\mathcal{X}_{-j}=(\boldsymbol{o}_{-j}^{k_2},\boldsymbol{a}_{-j}^{k_2})$ from $\mathcal{X}_{j}$, which are the top $K_2$ frequent observation-action pairs in the minibatch $\mathcal{X}_{j}$ after removing $o_j$ and $a_j$ from $\mathcal{X}_{j}$. 
Then we form the sampled trajectories by combining $(o_j,a_j)$ in $\mathcal{X}_{j}$ and $(\boldsymbol{o}_{-j},\boldsymbol{a}_{-j})$ in $\mathcal{X}_{-j}$ as $\mathcal{D}=(\boldsymbol{o}^{k_1k_2},\boldsymbol{a}^{k_1k_2},R^{k_1k_2},\boldsymbol{o}'^{k_1k_2})$. 
To obtain the action-values for clustering, we query the critic networks with $\mathcal{D}$ as the input to get the $\hat{Q}_{\omega}(o_j,\boldsymbol{o}_{-j},a_j, \boldsymbol{a}_{-j})$, which approximates action-value vectors $\bar{Q}^*(o_j)$ defined in Eq. (\ref{equ:vector_re_write_Q_1}) and Sec.~\ref{sec:theory}'s illustrative example. We use $\bar{Q}^*(o_j)$ instead of $\hat{Q}_{\omega}$ in the following part to be consistent with the theoretical results. 
The message $m_j=g_{\xi_j}(o_j)$ is updated by minimizing a Regularized Information Maximization (RIM) loss function~\cite{imsat} $\mathcal{L}(g_{\xi_j})$ in terms of $\bar{Q}^*(o_j)$:
\begin{equation} \label{eq: commloss}
    \setlength{\abovedisplayskip}{1pt}
    \setlength{\belowdisplayskip}{1pt}
    \begin{aligned}
&\mathcal{L}(g_{\xi_i}) = L_{CD} - \lambda L_{MI}, \\
&L_{CD}= \sum_{p=1}^{K_1}\sum_{q \in N_{K_3}(p)} \left[D_{cos}(\bar{Q}^*(o_j^p),\bar{Q}^*(o_j^q)\right]\|m_{j}^p-m_{j}^q\|^2, \\
&L_{MI} = I(o_{j};m_{j})=H(m_j)-H(m_j|o_j),
    \end{aligned}
\end{equation}
where $L_{CD}$ is a clustering loss in the form of  Locality-preserving loss~\cite{localitypreservingloss}. It preserves the locality of the clusters by pushing nearby data points of action-value vectors together. Inside $L_{CD}$, $o_j^p \in \mathcal{X}_j, p=1,\dots,K_1$ is the sampled observation $o_j$, $N_{K_3}(p)$ is the set of the $K_3$ nearest neighbors of $\bar{Q}^*(o_j^p)$, with $D_{cos}$ (the cosine-distance between $\bar{Q}^*(o_j^p)$ and its neighbor $\bar{Q}^*(o_j^q)$) as the metric to define the neighbors. The mutual information loss $L_{MI}$ measures the mutual information between observation $o_{j}$ and message $m_{j}$. 
Here we measure the mutual information loss as the difference between the marginal entropy $H(m_j)$ and conditional entropy $H(m_j|o_j)$ to ensure uniformly-sized clusters and more unambiguous cluster assignments~\cite{imsat}.
Action-value vectors are normalized and processed through the activation function for improved clustering outcomes. The choice of activation function is discussed in Sec.~\ref{sec:evaluation}. The Pseudo-Code of training the RGMComm is in Appendix B.



\section{Experiments}
\label{sec:evaluation}

We test RGMComm on continuous state space Dec-POMDP problems (using Multi-Agent Particle Environment~\cite{maddpg}). We compare it against strong baselines: 
\textbf{(1).} \textbf{Continuous communication} using SARNet~\cite{sarnet}, which employs memory-based attention for uninterrupted messages.
\textbf{(2).} One-hot \textbf{discrete communication} using CommNet~\cite{commnet}, MADDPG~\cite{maddpg}, IC3Net~\cite{ic3net}, TarMAC~\cite{tarmac}, and SARNet~\cite{sarnet}. We enable their discrete communication mode, where agents send discrete unstructured vectors. These baselines use Gumbel softmax for gradients and backpropagation.
\textbf{(3).} \textbf{Vector-Quantized} Variational Information Bottleneck (VQ-VIB)~\cite{tucker2022trading}, the recent method using autoencoder and quantization for communication.
All baselines train from scratch with the same hyperparameters as RGMComm, repeated 3 times with different seeds. Appendix C has more details on baselines, architectures, hyperparameters, and environments.


\begin{figure*}[ht]
     \centering
     \begin{subfigure}[b]{0.325\textwidth}
         \centering
         \includegraphics[width=\textwidth]{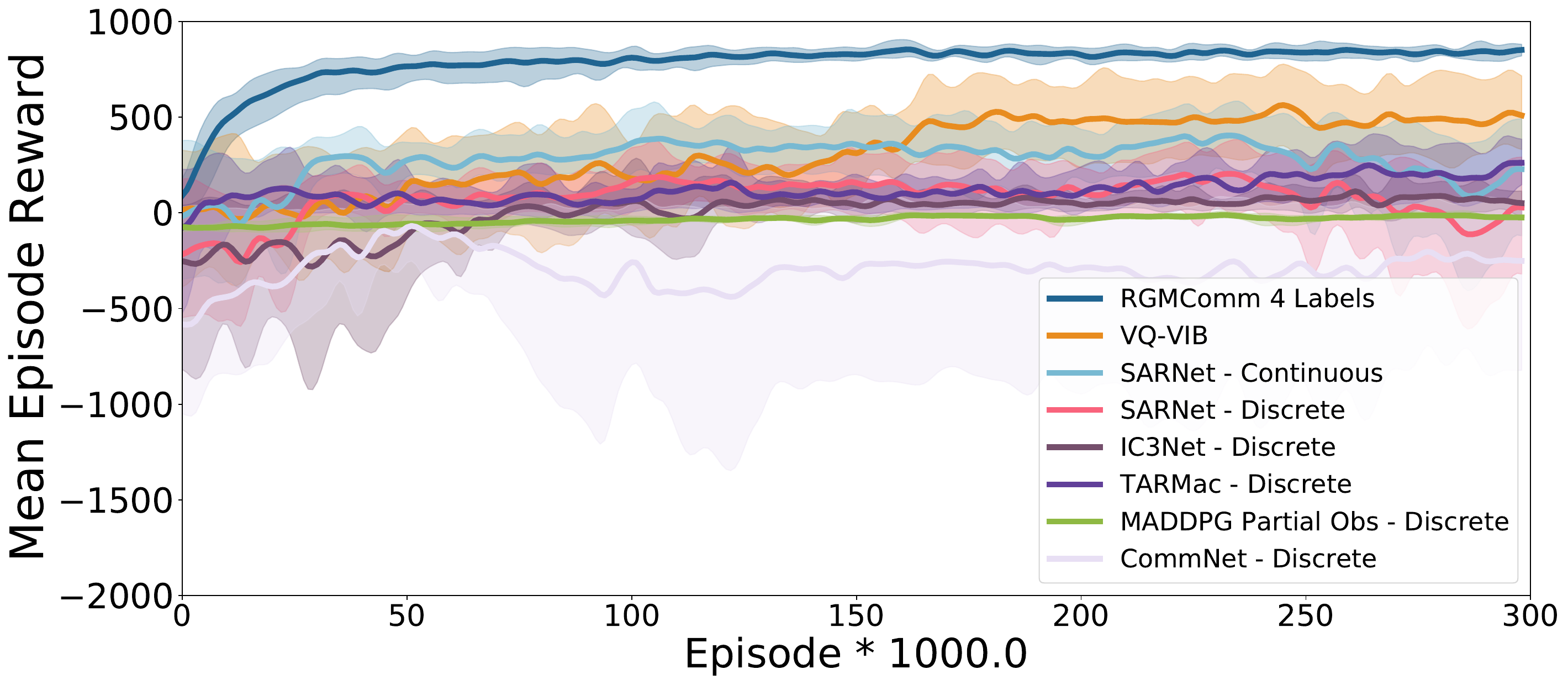}
         \caption{Predator-Prey with 2 agents}
         \label{fig:simple_tag_2_baseline}
     \end{subfigure}
     \hfill
     \begin{subfigure}[b]{0.325\textwidth}
         \centering
         \includegraphics[width=\textwidth]{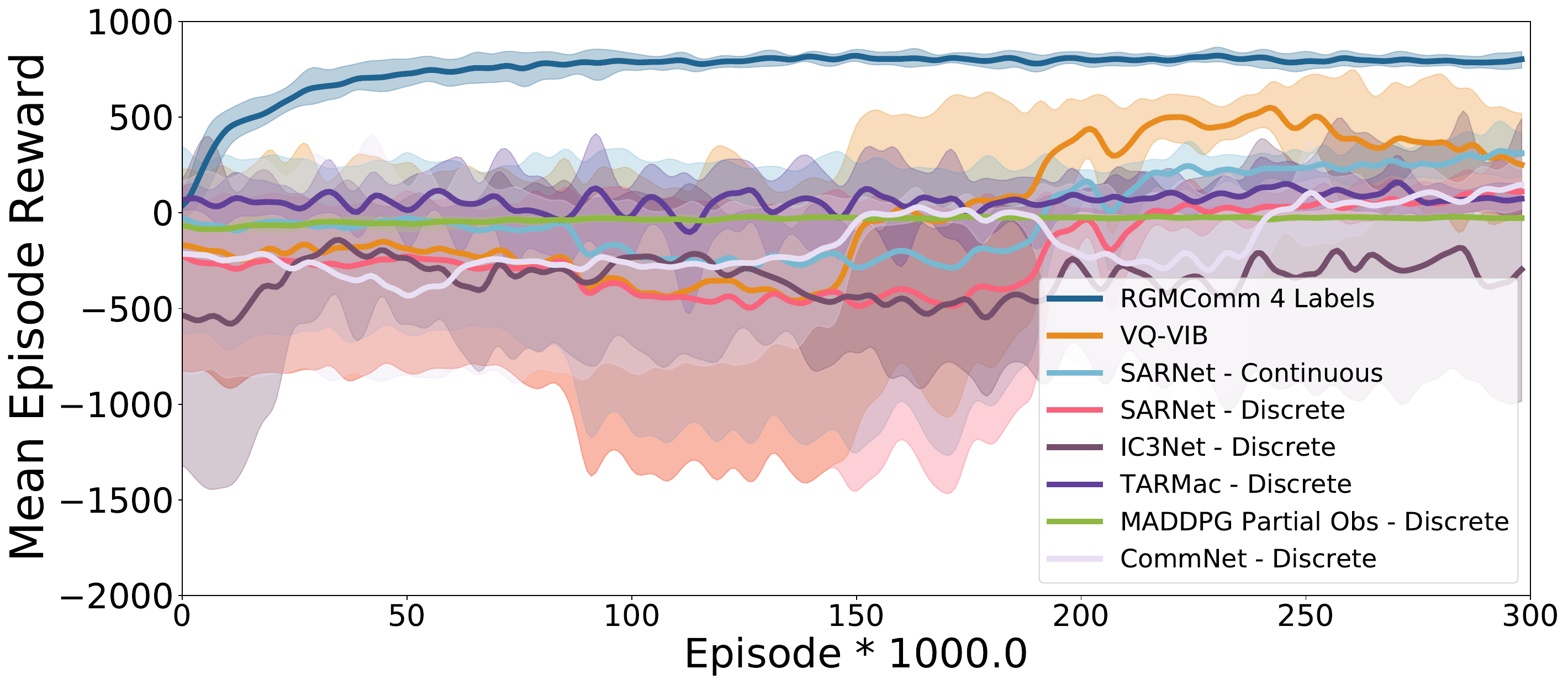}
        \caption{Predator-Prey with 3 agents}
         \label{fig:simple_tag_3_baseline}
     \end{subfigure}
     \hfill
     \begin{subfigure}[b]{0.325\textwidth}
         \centering
         \includegraphics[width=\textwidth]{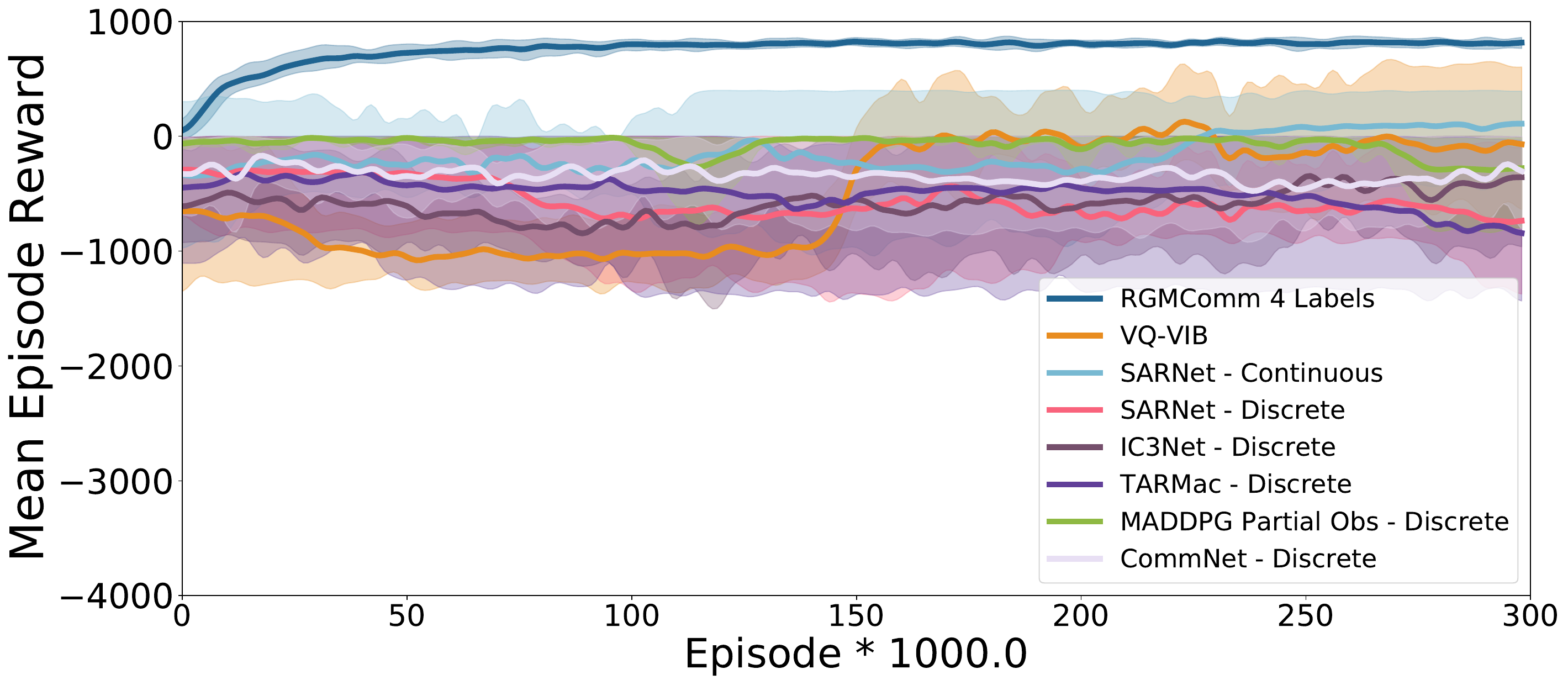}
        \caption{Predator-Prey with 6 agents}
         \label{fig:simple_tag_6_baseline}
     \end{subfigure}
     \hfill
     \begin{subfigure}[b]{0.325\textwidth}
         \centering
         \includegraphics[width=\textwidth]{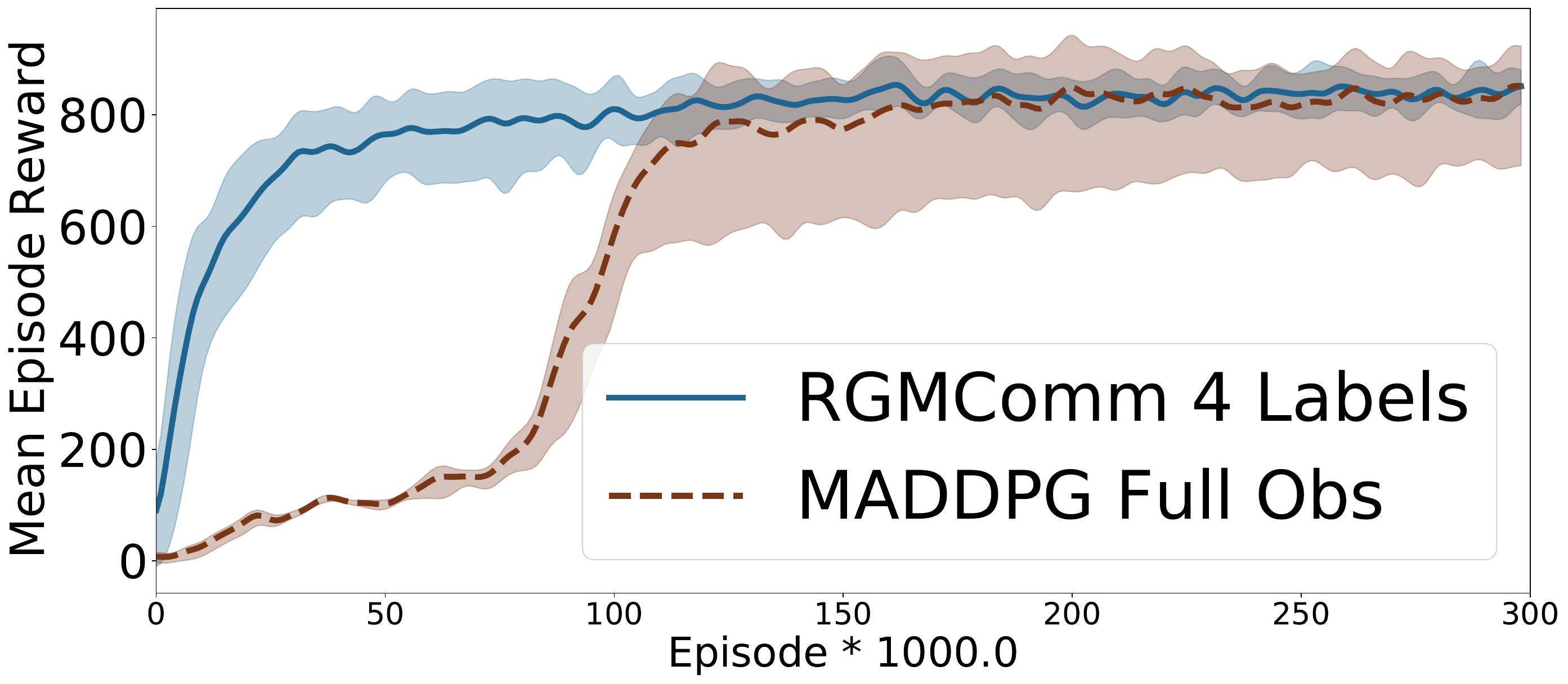}
        \caption{Predator-Prey with 2 agents}
         \label{fig:simple_tag_2_full}
     \end{subfigure}
     \hfill
     \begin{subfigure}[b]{0.325\textwidth}
         \centering
         \includegraphics[width=\textwidth]{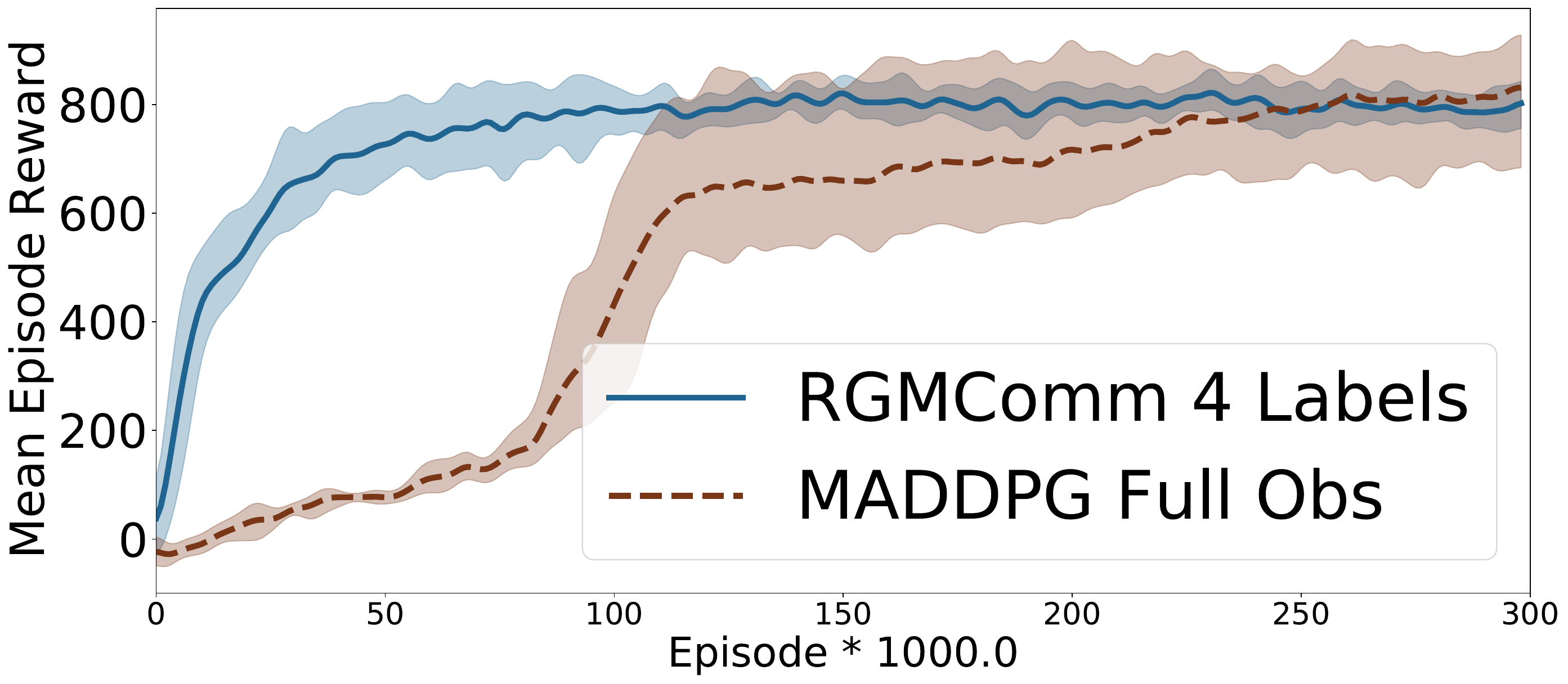}
        \caption{Predator-Prey with 3 agents}
         \label{fig:simple_tag_3_full}
     \end{subfigure}
     \hfill
     \begin{subfigure}[b]{0.325\textwidth}
         \centering
         \includegraphics[width=\textwidth]{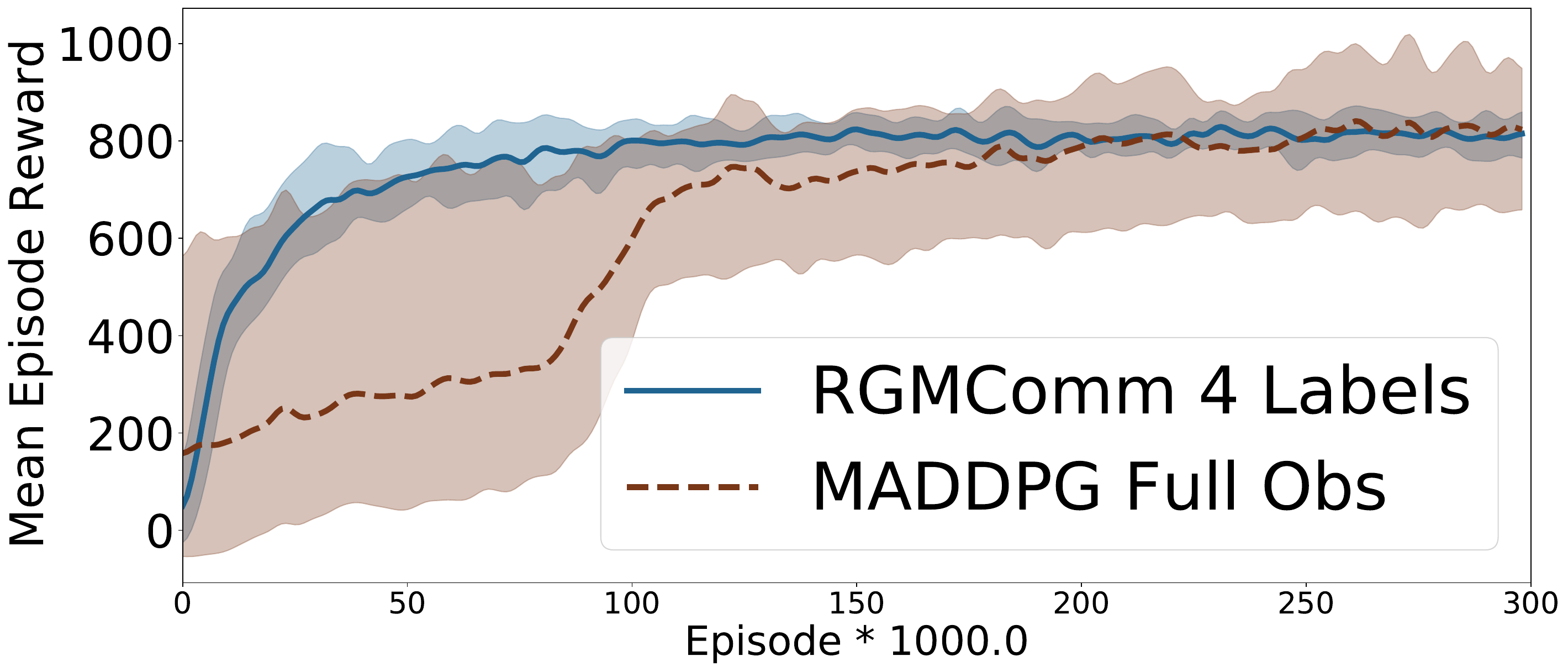}
        \caption{Predator-Prey with 6 agents}
         \label{fig:simple_tag_6_full}
     \end{subfigure}

    \caption{Evaluation on $n$-RGMComm-trained predators in Predator-Prey tasks: (a)-(c) Comparing RGMComm with baselines with communication: RGMComm (blue curve) converges to a higher mean episode reward than all the baselines. (d)-(f) Comparing RGMComm with full-observability policy: RGMComm (blue curve) achieves nearly optimal mean episode reward (brown dashed curve) in all scenarios with varying numbers of agents, which illustrates its ability to minimize the return gap. }
\label{fig:simple_tag_2_3_6} 
\vspace{-0.2in}
\end{figure*}

\textbf{Predator-Prey: }In the Predator-Prey scenario, predators are trained to collaborate to surround and seize \textbf{prey who move randomly}. 
We trained RGMComm with different alphabet sizes $|M|$ and $N=2,3,6$ \textbf{predators}. 
Figures \ref{fig:simple_tag_2_baseline} to \ref{fig:simple_tag_6_baseline} show that RGMComm outperforms all baselines with 2, 3, and 6 cooperative \textbf{predator} agents. The figures display learning curves of 300,000 episodes in terms of the mean episode reward, averaged over all evaluating episodes (10 episodes evaluated every 1000 episodes).
Figures \ref{fig:simple_tag_2_full} to \ref{fig:simple_tag_6_full} show the learning curves of average returns under our RGMComm policy compared with an ideal full-observability policy $\pi^*$.
RGMComm achieves near-optimal mean episode rewards in all scenarios with varying numbers of agents, which demonstrates its ability to minimize the return gap. 
Additionally, the message generation training process uses an online clustering algorithm that does not require a large volume of data, leading to faster convergence that scales well with more agents.

\begin{figure}[th!]
\centering
\vspace{-0.1in}
\includegraphics[width=0.45\textwidth]{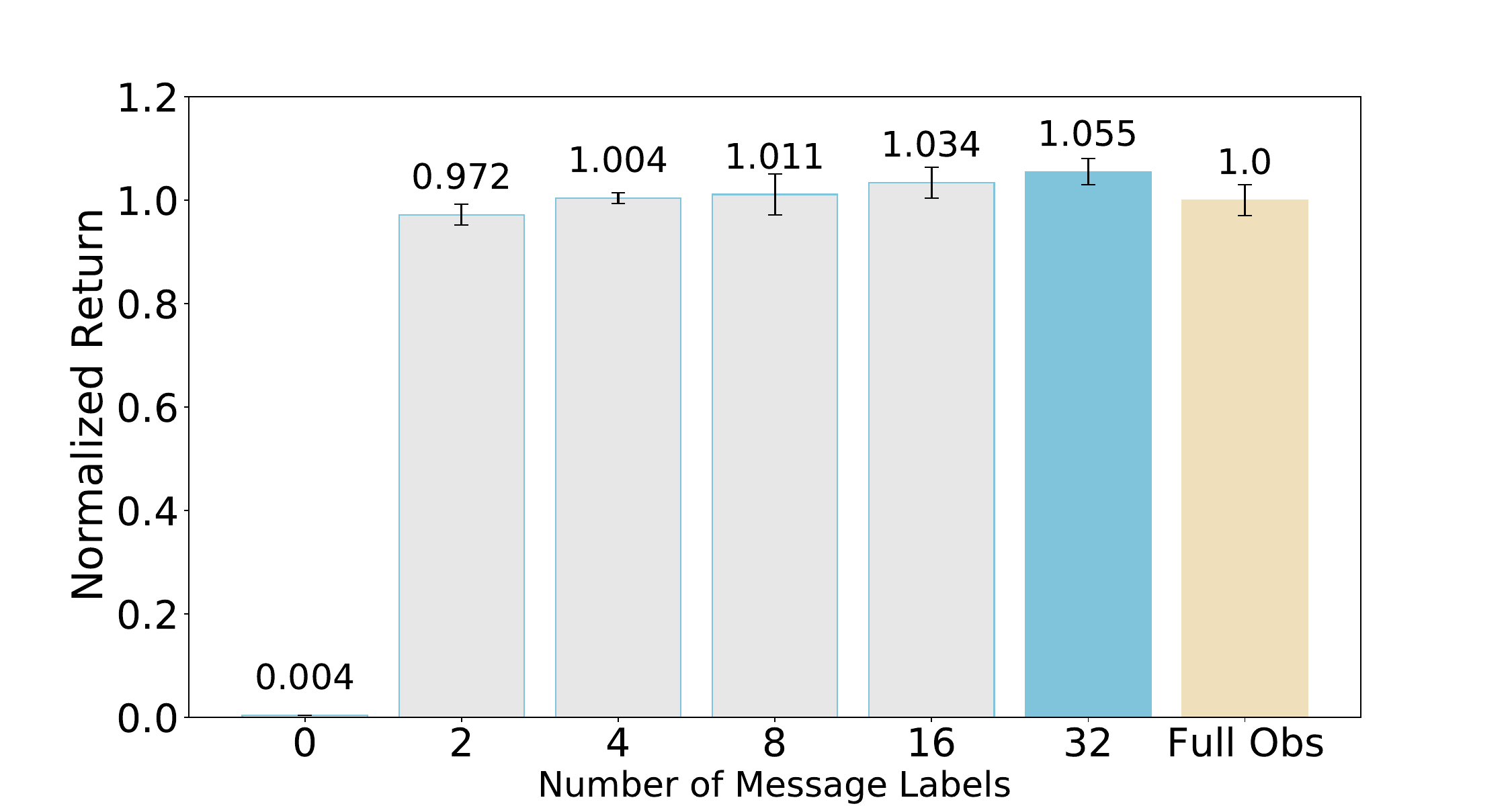}
\caption{The normalized mean episode reward increases as the total number of message labels increases.}
\label{fig:simple_tag_2_reward_labels}
\vspace{-0.1in}
\end{figure}

Fig~\ref{fig:simple_tag_2_reward_labels} shows how the total number of message labels affects RGMComm's performance. We compare the normalized mean episode reward achieved by RGMComm with $|M|=0,2,4,8,16,32$ message labels. As expected, the mean episode reward increases with the number of message labels. Remarkably, with only $|M|=2$ message labels (i.e., 1-bit communication), RGMComm achieves nearly optimal mean episode reward. With more than $|M|=4$ message labels, RGMComm's reward exceeds that of the policy with full observability, 
since the message generation function learned from the critic provides a succinct, discrete representation of the optimal action-value structure, leading to less noisy communication signals and allowing agents to discover more efficient decision-making policies conditioned on the message labels.


\begin{figure*}[ht]
  \centering
  \begin{subfigure}[b]{0.325\textwidth}
    \centering
    \includegraphics[width=\textwidth]{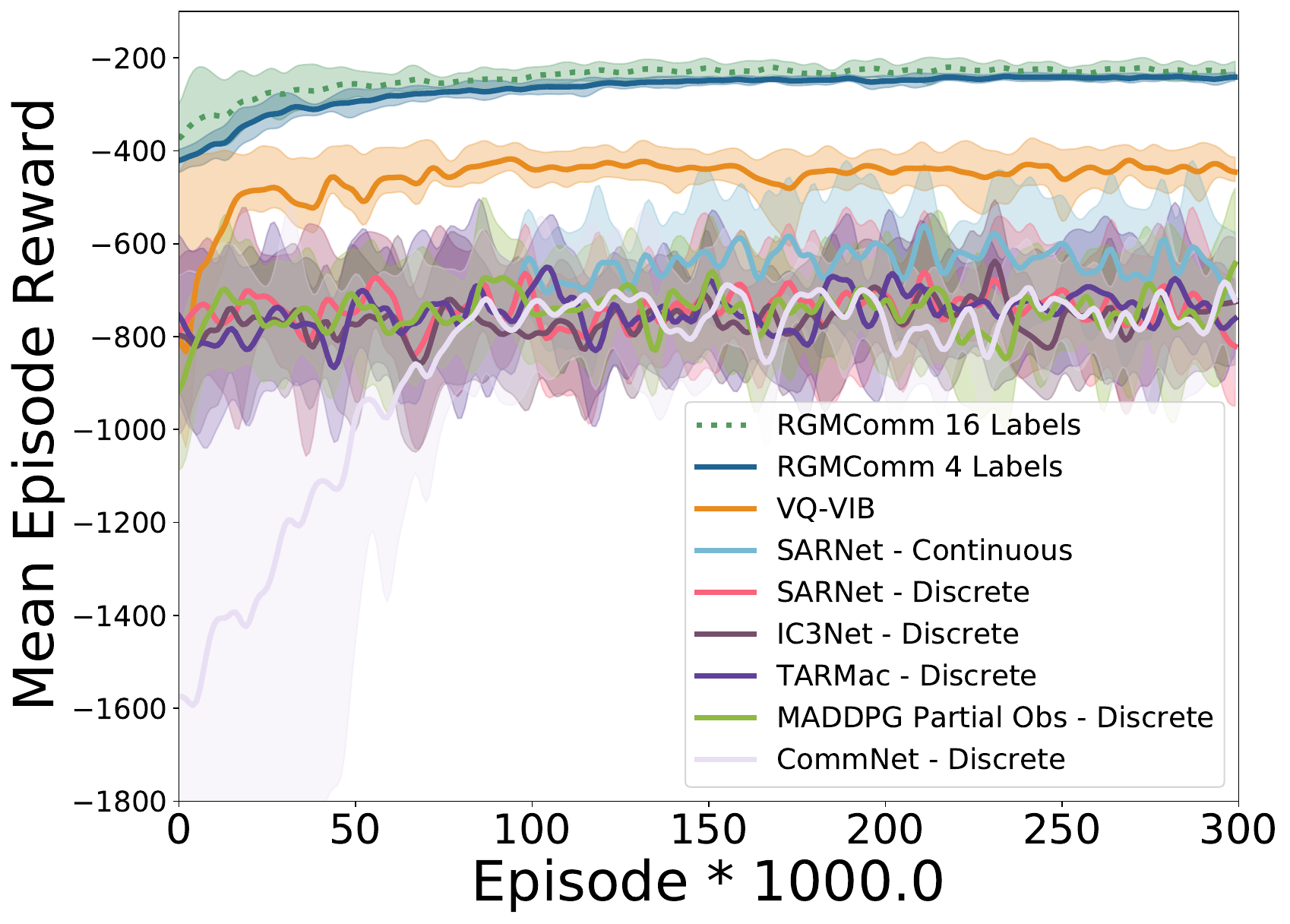}
    \caption{RGMComm using 4 and 16 message labels both have higher convergence values than all baselines with partial observability in the Cooperation Navigation task.}
    \label{fig:simple_spread_2_4labels_16labels_baselines}
  \end{subfigure}
  \hfill
  \begin{subfigure}[b]{0.325\textwidth}
    \centering
    \includegraphics[width=\textwidth]{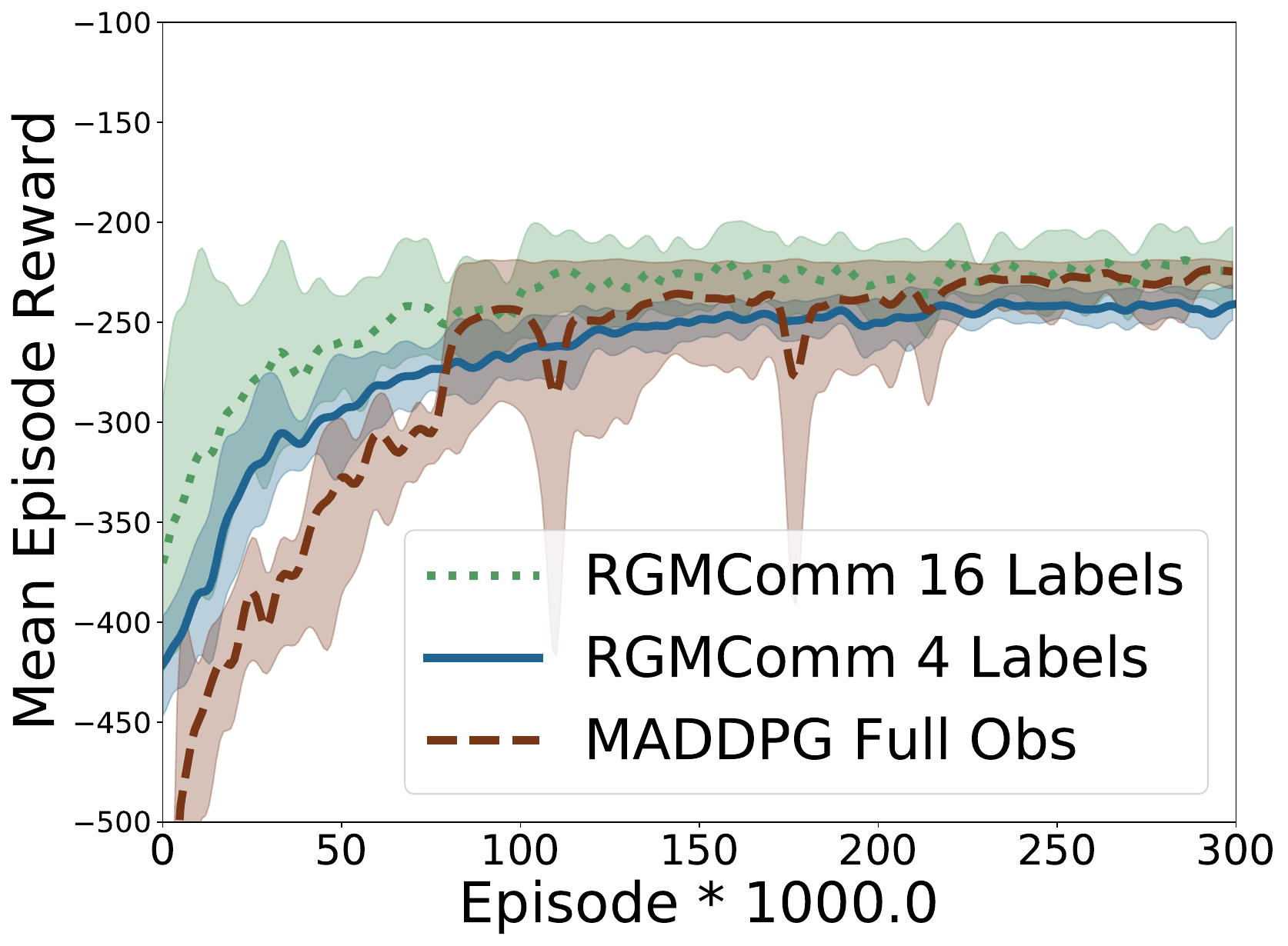}
    \caption{ RGMComm using 4 or 16 message labels both lead to almost zero return gap and RGMComm achieves even higher returns than full observability policy with 16 labels.}
    \label{fig:simple_spread_2_4labels_16labels_full}
  \end{subfigure}
  \begin{subfigure}[b]{0.335\textwidth}
    \centering
    \includegraphics[width=1.05\textwidth]{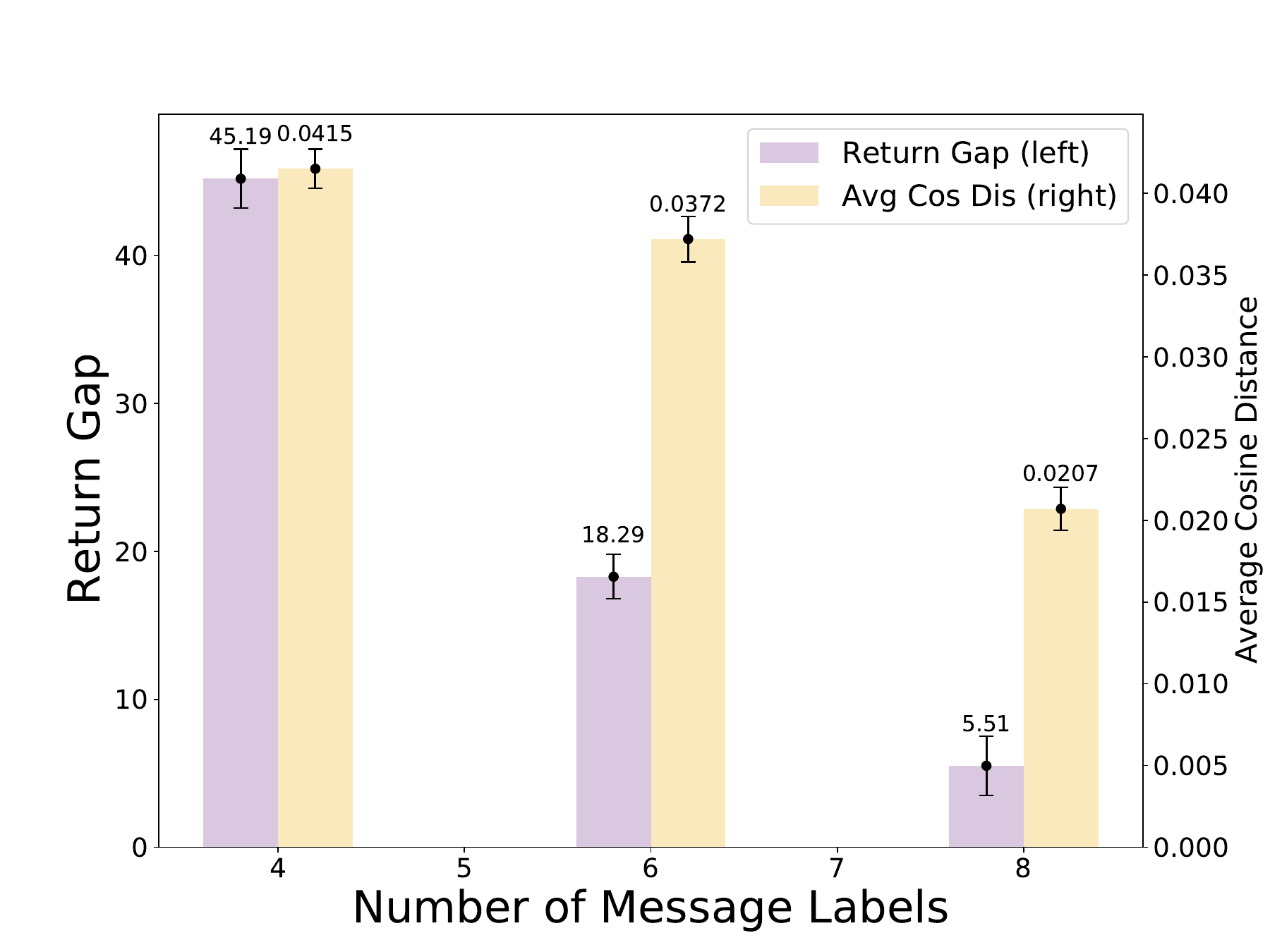}
    \caption{The return gap (left) is bounded by average cosine-distance (right) and diminishes as average cosine-distance (right) decreases due to the use of more message labels.}
    \label{fig:simple_spread_2_regret_cosdis}
  \end{subfigure}
  \caption{Cooperation Navigation Experiments}
  \vspace{-0.1in}
\end{figure*}

\textbf{Cooperative Navigation: }We train RGMComm using 4 and 16 message labels for Cooperative Navigation, which tasks agents with collaboratively navigating to specific targets without collisions.
Fig.~\ref{fig:simple_spread_2_4labels_16labels_baselines} shows that RGMComm converges faster and to a much higher reward than all the baselines. Figure~\ref{fig:simple_spread_2_4labels_16labels_full} shows that RGMComm using four message labels leads to almost zero return gap and achieves nearly-optimal returns in both cases. With 16 message labels, RGMComm obtains a higher reward than the policy with full-observability, demonstrating that discrete communication allows more efficient policies to be learned in POMDP.

Figure~\ref{fig:simple_spread_2_regret_cosdis} plots the average cosine-distance $\epsilon$ across clusters and the corresponding return gap (between RGMComm and the full-observability policy) using 4, 6, and 8 message labels. 
The numerical results justify Thm.~\ref{thm:thm_1}'s analysis: the return gaps are indeed bounded by $O(\sqrt{\epsilon}nQ_{\rm max})$, and the average return gap diminishes as the average cosine-distance over all clusters decreases due to using more message labels, validating Remark~\ref{remark:label_M}.


\begin{figure}[th]
     \centering
     \begin{subfigure}[b]{0.22\textwidth}
         \centering
         \includegraphics[width=1.04\textwidth]{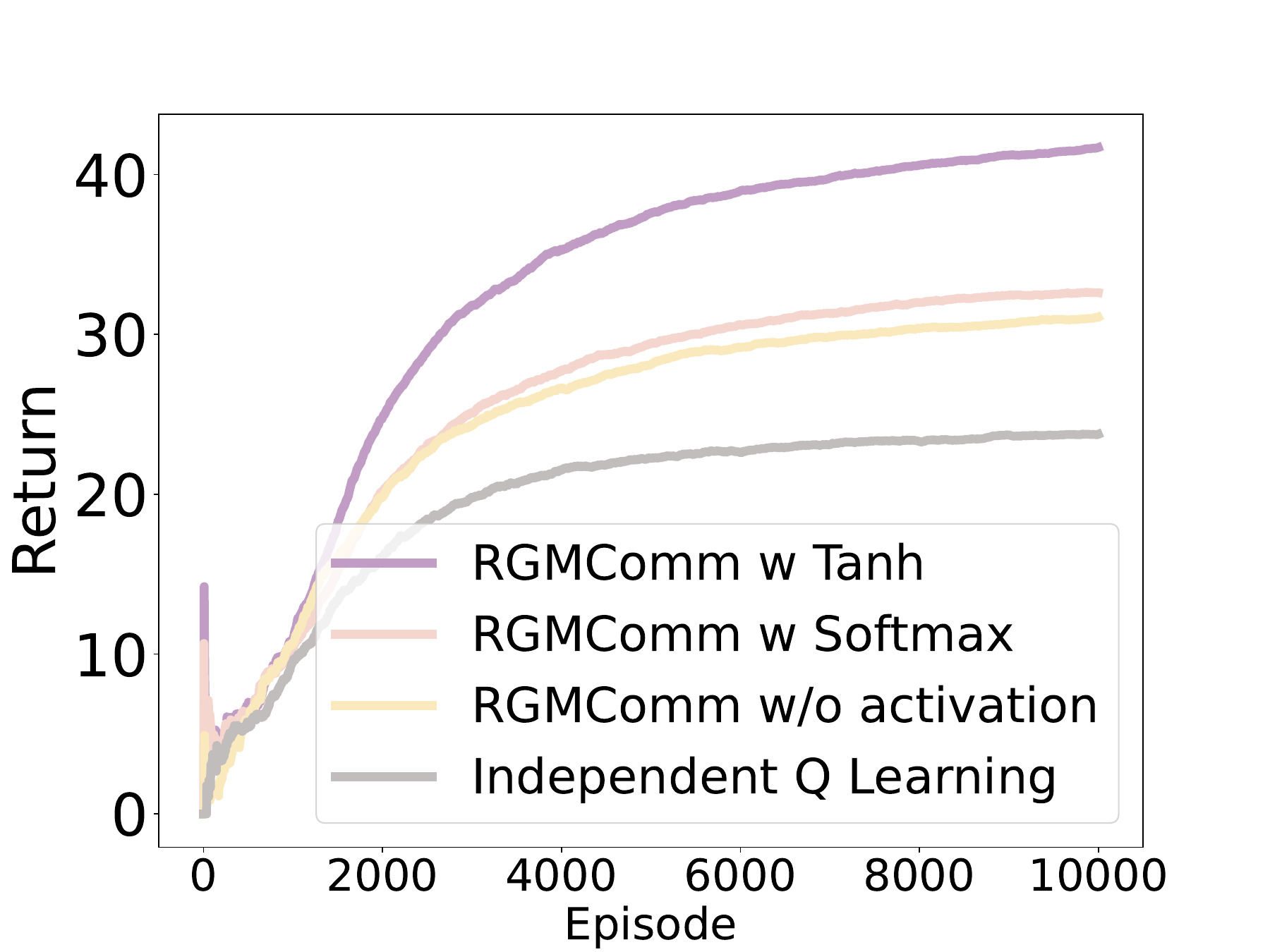}
         \caption{RGMComm with baselines}
         \label{fig:maze_results_baselines}
     \end{subfigure}
     \hfill
     \begin{subfigure}[b]{0.22\textwidth}
         \centering
         \includegraphics[width=1.04\textwidth]{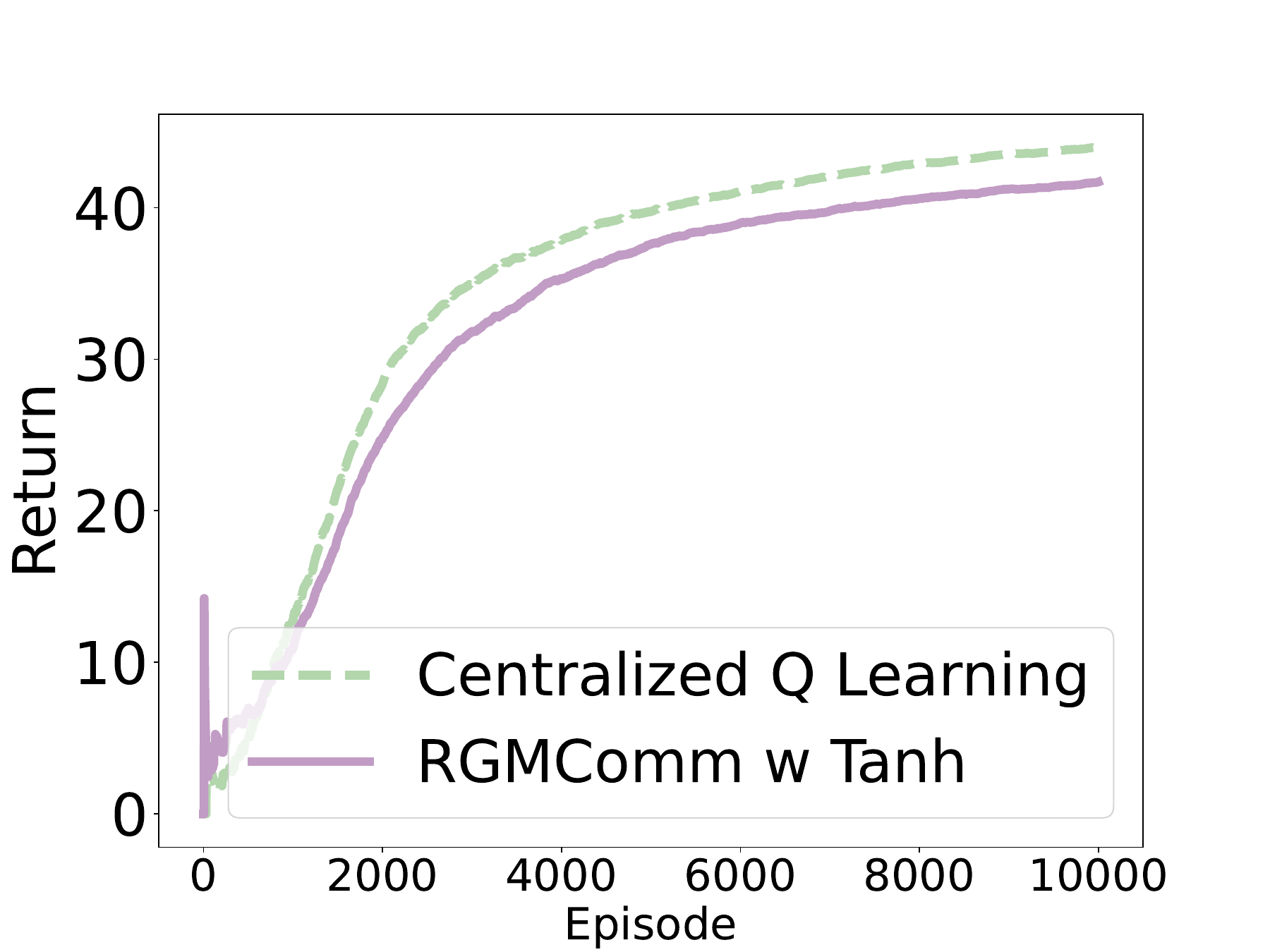}
        \caption{RGMComm with full obs}
         \label{fig:maze_results_full}
     \end{subfigure}
    \caption{(a): RGMComm with Tanh activation function (purple line) outperforms all baselines with other activation functions; (b): RGMComm with Tanh (purple line) converges to almost the optimal mean episode reward achieved by centralized Q learning (green dashed line).}
\vspace{-0.20in}
\end{figure}


\textbf{Ablation Study: }We conduct an ablation study on RGMComm, investigating the effect of different activation functions on normalizing action value vectors $\bar{Q}^*(o_j)$. This choice impacts the clustering's distance metrics. We compare the Softmax, hyperbolic tangent (Tanh), and no activation function approaches in our RGMComm algorithm against two baselines: independent Q learning (no communication) and centralized Q learning (with full observability). Figures~\ref{fig:maze_results_baselines} and~\ref{fig:maze_results_full} demonstrate RGMComm with Tanh achieves higher returns, approaching near-optimal performance similar to the full observability algorithm. This result stems from the S-shape functions pushing larger action-values towards 1. This behavior assists the cosine-distance-based RGMComm to group action values likely maximized at the same actions.


\begin{figure}[th]
     \centering
     \begin{subfigure}[b]{0.2\textwidth}
         \centering
         \includegraphics[width=\textwidth]{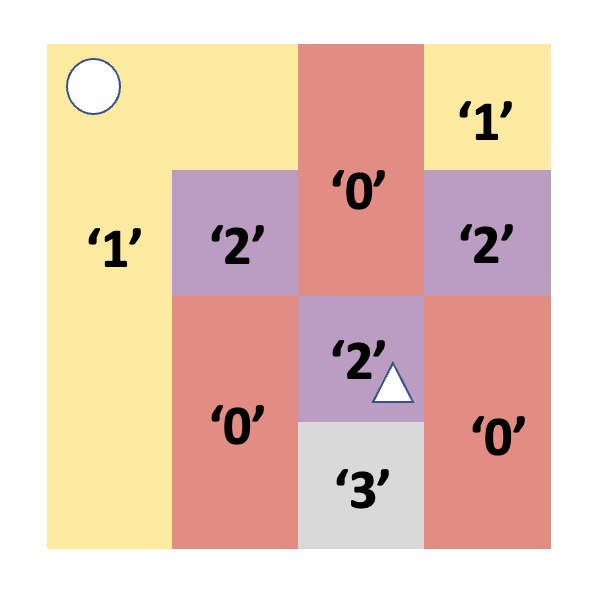}
         \caption{Label-position (agent 1)}
         \label{fig:agent1_maze_position_label}
     \end{subfigure}
     \hfill
     \begin{subfigure}[b]{0.2\textwidth}
         \centering
         \includegraphics[width=\textwidth]{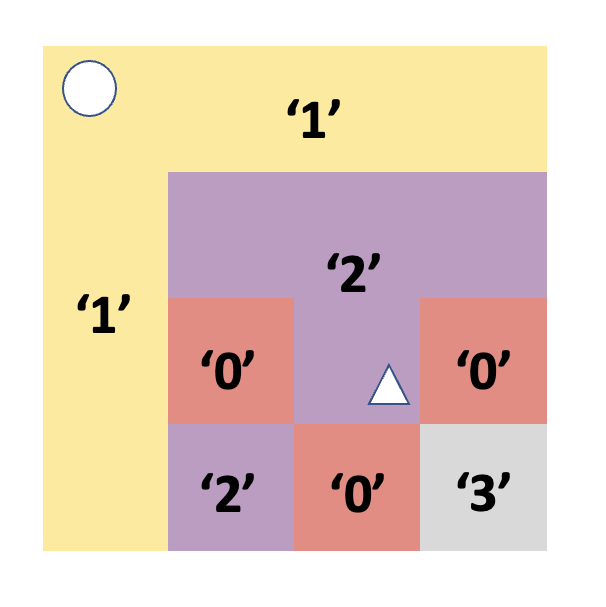}
        \caption{Label-position (agent 2)}
         \label{fig:agent2_maze_position_label}
     \end{subfigure}
     \begin{subfigure}[b]{0.21\textwidth}
         \centering
         \includegraphics[width=\textwidth]{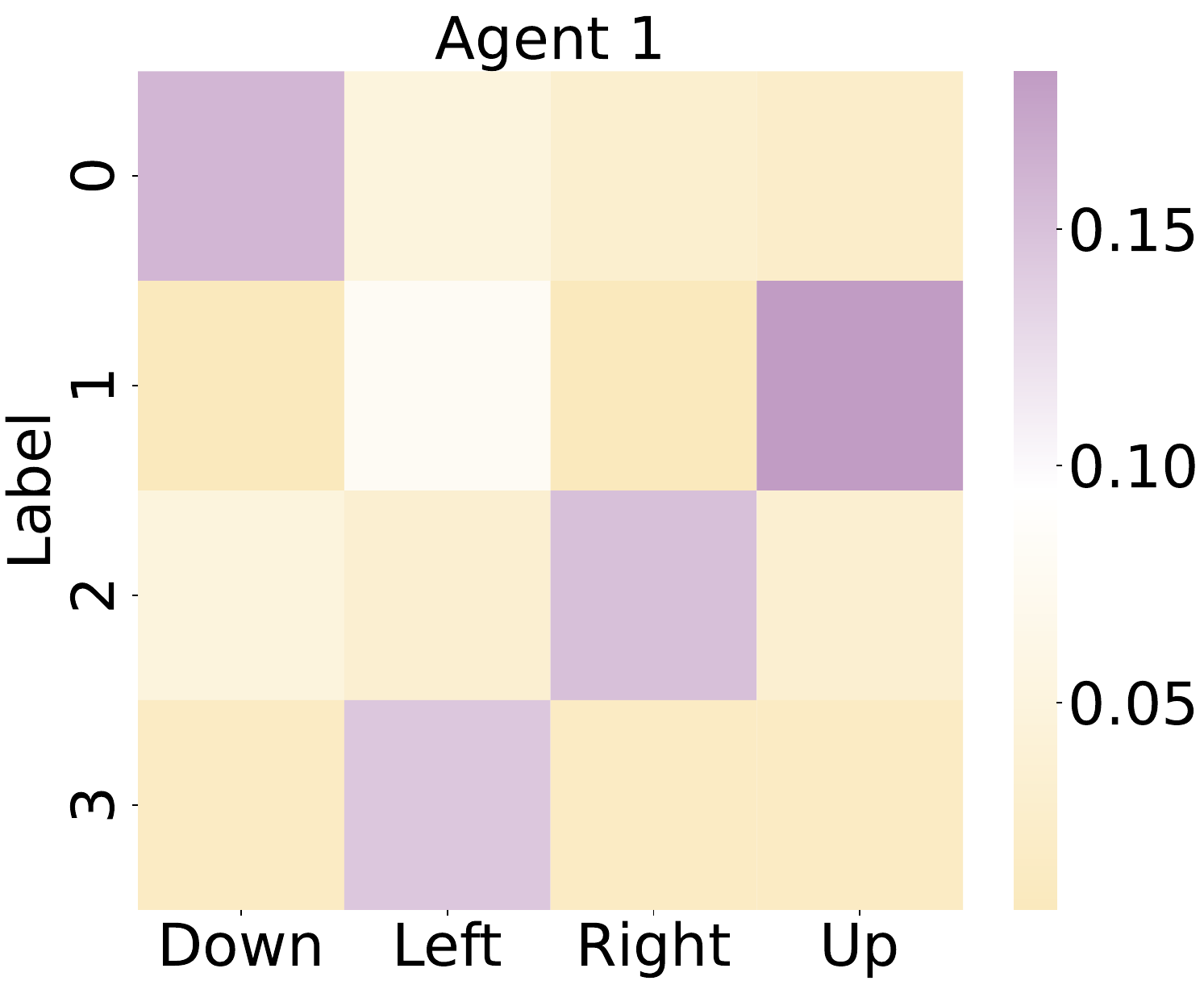}
         \caption{Label-action (agent 1)}
         \label{fig:agent1_maze_action_label}
     \end{subfigure}
     \hfill
     \begin{subfigure}[b]{0.22\textwidth}
         \centering
         \includegraphics[width=\textwidth]{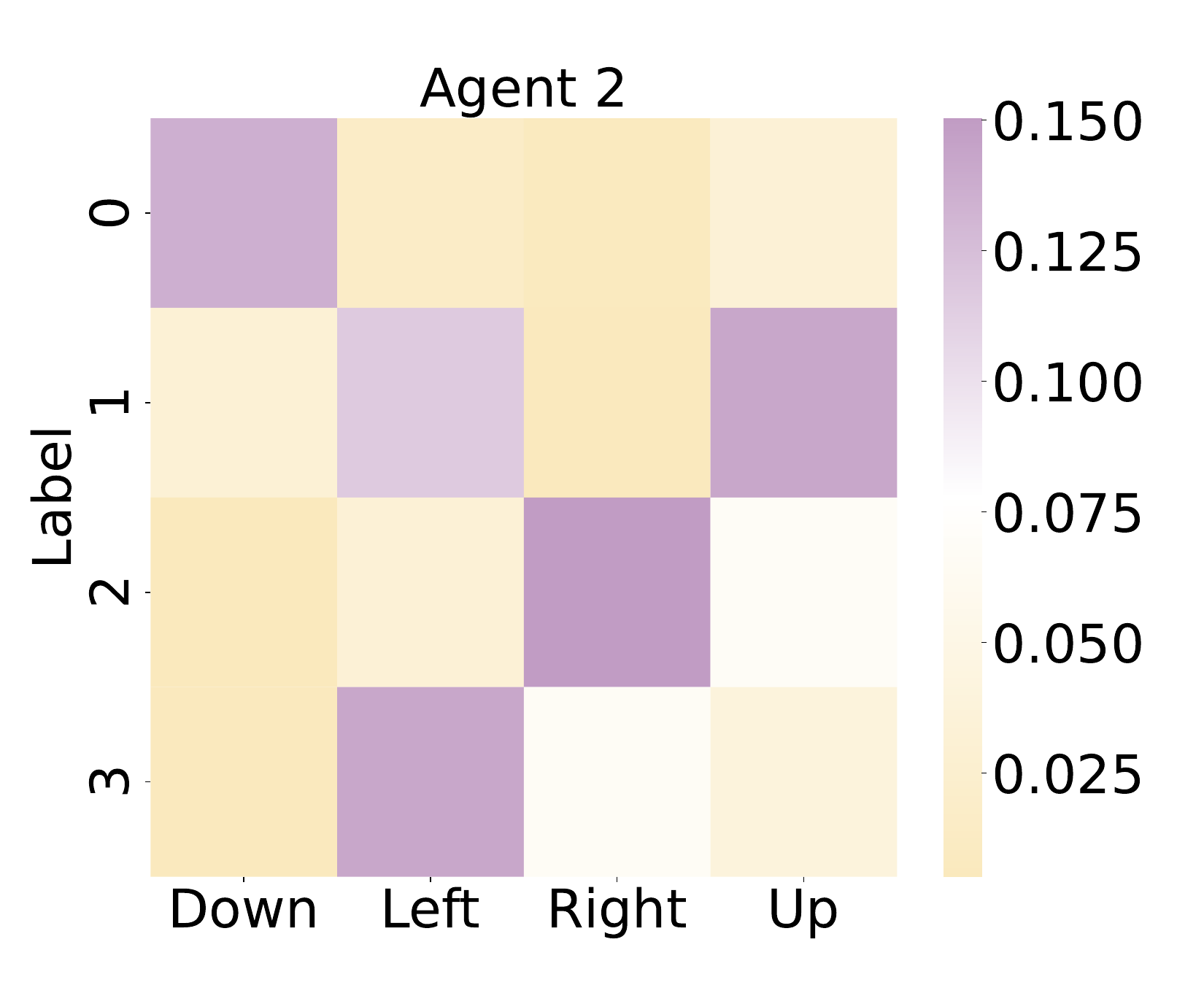}
         \caption{Label-action (agent 2)}
         \label{fig:agent2_maze_action_label}
     \end{subfigure}
    \caption{(a)-(b): agents are more likely to send label `1' when closer to higher-reward target (circle), while more likely to send label `2' when closer to lower-reward target (triangle); (c)-(d): both agents are more likely to take action `down', `up', `right', `left' conditioned receiving label `0', `1', `2', and `3', respectively.}
    \label{fig:maze_label}
\vspace{-0.2in}
\end{figure}

\textbf{Communication Message Interpretations: }The discrete message labels in RGMComm are easily interpreted, demonstrated in a grid-world maze environment for visual clarity. In Figures~\ref{fig:agent1_maze_position_label} and~\ref{fig:agent2_maze_position_label}, message label frequencies emitted by agents in different positions during training are visualized. Positions are depicted using colors: `red', `yellow', `purple', and `grey' represent higher probabilities of sending message labels `0', `1', `2', and `3', respectively. The proximity to high-reward (circle) or low-reward (triangle) targets influences agents to send specific labels.


Figures~\ref{fig:agent1_maze_action_label} and \ref{fig:agent2_maze_action_label} visualize the correlation between actions and received message labels for RGMComm agents. Actions tend to align with labels, facilitating effective coordination. For instance, integrating the four figures in Fig.~\ref{fig:maze_label}, we observe that when an RGMComm agent approaches the upper left high-reward target, it transmits label `1' to guide other agents to move `up' or `left', effectively steering them towards this target. In summary, RGMComm provides interpretable discrete communication, enabling task-specific interpretations during training.

\section{Conclusion}

This paper introduces a discrete communication framework for Multi-Agent Reinforcement Learning (MARL) employing any finite-size discrete alphabet. Our approach views message generation as an online clustering problem and quantifies the optimal return gap between an ideal policy and a communication-enabled policy with a closed-form upper-bound. We propose a novel class of MARL communication algorithm, RGMComm, designed to minimize the return gap using few-bit messages. RGMComm significantly outperforms state-of-the-art baselines and achieves nearly-optimal returns. 
Further research will delve into more use cases.



\newpage
\section*{Acknowledgements}
This work was funded by ONR grants N00014-23-1-2532, N00014-23-1-2850, CISCO Research Award and also in part by the Office of Naval Research under Grant W911NF19110036. Any opinions, findings, and conclusions or recommendations expressed in this material are those of the author(s) and do not necessarily reflect the views of the Office of Naval Research.

\bigskip
\bibliography{aaai24}

\newpage
\appendix
\onecolumn
\section{Theoritical Content} \label{appendix:proofs}
\subsection{Notations}
We summarize all notations here for ease of reading.

\begin{itemize} 
    \item $m_j$: $m_{j} = g(o_j)$ is the message labels generated by RGMComm algorithm with observation $o_j$ as the input.
    \item $\boldsymbol{m_{-i}}$: $\boldsymbol{m_{-i}}=\{m_{j},\forall j \neq i\}$ be the collection of all messages received by agent $j$.   
    \item $\pi^*$: $\pi^*=[\pi_i^*(a_i|o_1,\ldots,o_n),\forall i]$ is the optimal joint policy of all agents with full observability.
    \item $\pi$: $\pi=[\pi_i(a_i|o_i,\boldsymbol{m_{-i}}),\forall i]$ is a communication-enabled, partially-observable policy. From the view of agent $i$, each unavailable observation of other agents $j \neq i$ is replaced by a message label $m_j=g(o_j)$ generated based on local observation $o_j$ of agent $j$ and send to agent $i$.
    \item $J(\pi)$: $J(\pi) =\lim_{T \to \infty} (1/T) E_{\pi} [{\sum_{t=0}^T R_{t}}]$ is the average expected return of a policy $\pi$;
    \item $J(\pi^*) - J(\pi,g)$: is the return gap between an ideal policy $\pi^*=[\pi_i^*(a_i|o_1,\ldots,o_n),\forall i]$ with full observability and a partially-observable policy with communications $\pi=[\pi_i(a_i|o_i,\boldsymbol{m_{-i}}),\forall i]$ where message labels $\boldsymbol{m_{-i}}=\{m_{j}=g(o_j), \forall j \neq i\}$. The communication-enabled policy $\pi$ and communication function $g$ can be trained simultaneously.
    \item $\pi^*_{(j)}$ and $\pi_{(j)}$: 
these two policies are identical in the observability and policies of agents $i\neq j$, except that agent $j$'s policy is conditioned on communication messages $m_{j}$, while observability of all other agents $i\neq j$ remains unchanged. 
In detail, it means that $\pi^*_{(j)}$ is the optimal policy conditioned on complete observations $\boldsymbol{o} = [o_1,\dots,o_j,\dots,o_n]$ including observation $o_j$ of agent $j$, and $\pi_{(j)}$ is the policy conditioned on partial observations $\boldsymbol{o} = [o_1,\dots,o_{j_1}, m_{j}, o_{j+1}\dots,o_n]$ without knowing observation $o_j$ but including a communication message label $m_{j}=g(o_j)$ generated by agent $j$.
We use these two policies as auxiliary policies to help the theoretical proof.
\item $\bar{Q}^*(o_j)$: $\bar{Q}^*(o_j)=[\widetilde{Q}^*(o_{-j},o_j), \forall{o_{-j}}]$ is the action value vector corresponding observation $o_j$,
where $o_{-j}$ are the observations of all other agents.
\item $\widetilde{Q}^*(o_{-j},o_j)$: this is a vector of action values weighted by marginalized visitation probabilities $d_{\mu}^{\pi}(o_{-j}|o_j)$ and corresponding to different actions:
\begin{equation} \label{equ:vector_re_write_Q_1_0_app_notation}
   \begin{aligned}
    \widetilde{Q}^*(o_{-j},o_j)=[Q^*(o_{-j},o_j,\boldsymbol{a})\cdot d_{\mu}^{\pi}(o_{-j}|o_j), \forall \boldsymbol{a}], \forall{o_{-j}} \sim d_{\mu}^{\pi}(o_{-j}|o_j),\forall{\boldsymbol{a}} \sim \mu^{\pi}(o_{-j},o_j,\boldsymbol{a}).
    \end{aligned}
\end{equation}
Here $\mu^{\pi}(\boldsymbol{o},\boldsymbol{a})$ denotes the observation-action distribution given a policy $\pi$, which means $\mu^{\pi}(\boldsymbol{o},\boldsymbol{a}) = d_{\mu}^{\pi}(\boldsymbol{o})\pi(\boldsymbol{a}|\boldsymbol{o})$.
\item $d^{\pi}_{\mu}(\boldsymbol{o})$: this is the $\gamma$-discounted visitation probability under policy $\pi$ and initial observation distribution $\boldsymbol{o}(t=0) \sim \mu$,

\item $\bar{H}(m)$: $\bar{H}(m)=\sum_{o_j\sim m}\bar{d}_m(o_j)\cdot\bar{Q}^*(o_j)$ is the clustering center under each message $m$.
\item $\bar{d}_m(o_j)$: $\bar{d}_m(o_j)=d_{\mu}^{\pi}(o_j)/d_{\mu}^{\pi}(m)$ is the marginalized probability of $o_j$ in cluster $m$;
\item $d_{\mu}^{\pi}(m)$: this is the visitation probability of message $m$, it quantifies how frequently a message $m$ is encountered as an agent interacts with the environment according to the policy $\pi$ and environmental condition which is the initial observation distribution is $\mu$.

\item $\epsilon(o_j)$: consider two vectors, $\bar{Q}^*(o_j)$ and $\bar{H}(m)$. The cosine-distance between these vectors is represented by $\epsilon(o_j) = D_{cos}(\bar{Q}^*(o_j), \bar{H}(m))$. Cosine-distance measures the angular separation between vectors; a smaller value indicates greater similarity, while a larger value implies more dissimilarity. 

\item $\epsilon$: $\epsilon \triangleq \sum_m d_{\mu}^{\pi}(m) \sum_{o_j\sim m} \bar{d}_m(o_j) \cdot \epsilon(o_j)$, The term "average cosine-distance" $\epsilon$ is calculated across all clusters represented by different message labels $m$. Clusters are groups of action-value vectors with small cosine-distance. 
The calculation involves two nested summations:
\begin{enumerate}
    \item For each message "m", the outer summation $\sum_m$ calculates the visitation probability of that message under a certain policy and initial observation distribution, denoted as $d_{\mu}^{\pi}(m)$.

\item For each pair of a message "m" and its corresponding observation $o_j$ (denoted as $o_j\sim m$) in the cluster with clustering label $m$, the inner summation $\sum_{o_j\sim m}$ involves the product of two components:
\begin{enumerate}
    \item $\bar{d}_m(o_j)$, representing a weight associated with the observation $o_j$ based on message "m", which is the marginalized probability of $o_j$ in cluster "m";
\item  $\epsilon(o_j)$, which is the cosine-distance between vectors associated with $o_j$ in cluster "m".
\end{enumerate}
\end{enumerate}

\item $Q_{\rm max}$: it is the maximum absolute action value of $\bar{Q}^*(o_j)$ in each cluster as $Q_{\rm max}=max_{o_j}||\bar{Q}^*(o_j)||_2$,

\end{itemize}

\subsection{Proof Outline for the final Theoretical Results in Theorem 4.3}
The proof consists of the following major steps:

(\romannumeral1). We characterize the change in optimal expected average return $J(\pi^*)-J(\pi,g)$ for any two policies $\pi^*$ (under full observability) and $\pi$ (under partial observability with communication) based on the joint action-values under $\pi^*$. The result is stated as a Policy Change Lemma~\ref{lemma:1}.

(\romannumeral2). We apply Policy Change Lemma~\ref{lemma:1} to two auxiliary policies $\pi^*_{(j)}$ and $\pi_{(j)}$, 
where $\pi^*_{(j)}$ is the optimal policy conditioned on complete observations $\boldsymbol{o} = [o_1,\dots,o_j,\dots,o_n]$ including observation $o_j$ of agent $j$, and $\pi_{(j)}$ is the policy conditioned on partial observations $\boldsymbol{o} = [o_1,\dots,o_{j_1}, m_{j}, o_{j+1}\dots,o_n]$ without knowing observation $o_j$ but including a communication message label $m_{j}=g(o_j)$ generated by agent $j$.
The result (formulated in Lemma~\ref{lemma:regret_0}) allows us to quantify the impact of message-encoding agent $j$'s local observations $o_j$ as $m_j$, rather than having complete access to $o_j$, on the optimal expected average return gap.

(\romannumeral3). 
Finally, we construct a sequence of policies beginning from $\pi^{*}$, end with $\pi$, which has a process of changing by replacing each observation $o_j$ by label $m_j$, one at a time:
$\pi^{*}=[\pi_i{(a_i|o_1,o_2,\dots,o_n)}, \forall i$] $\to$ $[\pi_i{(a_i|m_1,o_2,\dots,o_n)}, \forall i$] $\to$ $[\pi_i{(a_i|m_1,m_2,\dots,o_n)},\forall i]$ $\to$ $\dots$ $\to$ $\pi$=$[\pi_i({a_i|m_1,m_2,\dots,m_n}),\forall i]$. 
Applying the results from step (\romannumeral2) (Lemma~\ref{lemma:regret_0}) for $n$ times, we can quantify the desired return gap between $\pi^*$ with full observability and communication-enabled policy $\pi$ in Thm.~\ref{thm:thm_1}. 

\subsection{Proofs for Lemma 4.1}
\label{appendix:app_proof_lemma1}

\begin{proof}
We prove the result in this lemma by leveraging observation-based state value function $V^{\pi}(\boldsymbol{o})$ in Eq.(2) and the corresponding action value function $Q^{\pi}(\boldsymbol{o},\boldsymbol{a})=\mathbb{E}_{\pi}\left[\sum_{i=0}^{\infty}\gamma^i \cdot R_{t+i} \Big|\boldsymbol{o}_t=\boldsymbol{o}, \boldsymbol{a}_{t}= \boldsymbol{a}\right]$ to unroll the Dec-POMDP. Here we consider all other agents $i\neq j$ as a conceptual agent denoted by $-j$.

\begin{equation} \label{equ:regret_1_app}
    \setlength{\abovedisplayskip}{1pt}
    \setlength{\belowdisplayskip}{1pt}
    \begin{aligned}
    &J(\pi^*)-J(\pi,g)\\
    &=  E_{\mu}(1-\gamma)[V^{*}(\boldsymbol{o}(0))-V^{\pi}(\boldsymbol{o}(0))],\\
    &=E_{\mu}(1-\gamma)[V^{*}(o_{-j}(0),o_j(0))-V^{\pi}(o_{-j}(0),o_j(0))],\\
    &=E_{\mu}(1-\gamma)[V^{*}(o_{-j}(0),o_j(0))-Q^{*}(o_{-j}(0),o_j(0),\boldsymbol{a}_0^{\pi}) + Q^{*}(o_{-j}(0),o_j(0),\boldsymbol{a}_0^{\pi}) - Q^{\pi}(o_{-j}(0),o_j(0),\boldsymbol{a}_0^{\pi})],\\
    &=E_{\mu}(1-\gamma)[\Delta^{\pi^{}}(o_{-j}(0),o_j(0),\boldsymbol{a}_0^{\pi{}})+(Q^{*}-Q^{\pi})],\\
    &=E_{\mu}(1-\gamma)[\Delta^{\pi^{}}(o_{-j}(0),o_j(0),\boldsymbol{a}_0^{\pi{}}) + E_{o_{-j}(1),o_j(1)\sim P(\cdot|o_{-j}(0),o_j(0),\boldsymbol{a}_0^{\pi})}[\gamma(V^{*}(o_{-j}(1),o_j(1)) -V^{\pi}(o_{-j}(1),o_j(1)))]],\\
    &=E_{\mu,o_{-j}(t),o_j(t)\sim P}(1-\gamma)[\sum_{t=0}^{\infty}\gamma^t \Delta(o_{-j}(t),o_j(t),\boldsymbol{a}_t^{\pi})|\boldsymbol{a}_t^{\pi}],\\
    &=E_{o_{-j},o_j \sim d_{\mu}^{\pi}, \boldsymbol{a}^{\pi}=\pi(o_{-j},g(o_j))}[\Delta(o_{-j},o_j,\boldsymbol{a}^{\pi})|\boldsymbol{a}^{\pi}], \\
    &=\sum_{m}\sum_{o_j \sim m}\sum_{o_{-j}}\Delta(o_{-j}(t),o_j(t),\boldsymbol{a}_t^{\pi})\cdot d_{\mu}^{\pi}(o_{-j},o_j), \\
    & =  \sum_{m} \sum_{o_j \sim m} \sum_{o_{-j}}[Q^{*}(o_{-j},o_j,\boldsymbol{a}_t^{\pi^*}) - Q^*(o_{-j},o_j,\boldsymbol{a}_t^{\pi})]\cdot d_{\mu}^{\pi}(o_{-j},o_j),\\
    &\le \sum_{m} \sum_{o_j \sim m} \sum_{o_{-j}}[Q^{*}(o_{-j},o_j,\boldsymbol{a}_t^{\pi^*}) - Q(o_{-j},o_j,\boldsymbol{a}_t^{\pi})]\cdot d_{\mu}^{\pi}(o_{-j},o_j),\\
    &\le \sum_{m} \sum_{\boldsymbol{o}\sim m} [Q^{{*}}(\boldsymbol{o},\boldsymbol{a}_t^{\pi^*}) - Q^{\pi}(\boldsymbol{o},\boldsymbol{a}_t^{\pi})]\cdot d_{\mu}^{\pi}(\boldsymbol{o}),
    \end{aligned}
\end{equation}
Step 1 is to use Eq.(3) in the main paper with initial observation distribution $\boldsymbol{o}(0)\sim \mu$ at time $t=0$ to re-write the average expected return gap; Step 2 is separate the observations as agent j and the other agents as a single conceptual agent $-j$, then $\boldsymbol{o}=(o_{-j},o_j)$; Step 3 and step 4 are to obtain the sub-optimality-gap of policy $\pi$, defined as $\Delta^{\pi}(\boldsymbol{o},\boldsymbol{a})=V^{*}(\boldsymbol{o})-Q^{*}(\boldsymbol{o},\boldsymbol{a}^{\pi}) = Q^{*}(\boldsymbol{o},\boldsymbol{a}^{\pi^*}) - Q^{*}(\boldsymbol{o},\boldsymbol{a}^{\pi})$, by substracting and plus a action-value function $Q^{*}(o_{-j}(0),o_j(0),\boldsymbol{a}_0^{\pi})$; Step 5 is to unroll the Markov chain from time $t=0$ to time $t=1$; Step 6 is to use sub-optimality-gap accumulated for all time steps to represent the return gap; Step 7 and step 8 is to absorb the discount factor $1-\gamma$ by  multiplying the $d^{\pi}_{\mu}(\boldsymbol{o})=(1-\gamma)\sum_{t=0}^{\infty} \gamma^t \cdot P(\boldsymbol{o}_t=\boldsymbol{o}|\pi,\mu)$ which is the $\gamma$-discounted visitation probability of observations $\boldsymbol{o}$ under policy $\pi$ given initial observation distribution $\mu$; Step 9 is to revert $\Delta^{\pi}(\boldsymbol{o},\boldsymbol{a})$ back to $Q^{*}(\boldsymbol{o},\boldsymbol{a}^{\pi^*}) - Q^{*}(\boldsymbol{o},\boldsymbol{a}^{\pi})$; Step 10 is to replace the second term $Q^{*}(\boldsymbol{o},\boldsymbol{a}^{\pi})$ with $Q^{\pi}(\boldsymbol{o},\boldsymbol{a}^{\pi})$. Since $Q^*(\boldsymbol{o},\boldsymbol{a}^{\pi})$ is larger than $Q^{\pi}(\boldsymbol{o},\boldsymbol{a}^{\pi})$, therefore, the inequality is valid.
\end{proof}

\subsection{Proof Outline for Lemma 4.2}
The proof contains the following major steps: 
(\romannumeral1). Recasting the communication problem into an online clustering problem;

(\romannumeral2). Re-writing the return gap derived in Policy Change Lemma~\ref{lemma:1} in vector terms using action-value vectors $\bar{Q}^*(o_j)$;

(\romannumeral3). Projecting action-value vectors towards cluster centers to quantify the return gap through orthogonal parts related to projection errors.

(\romannumeral4). Deriving an upper bound on the return gap by bounding the orthogonal projection errors using the average cosine-distance within each cluster.

\subsection{Proofs for Lemma 4.2}
\label{appendix:app_proof_lemma2}


\begin{proof}
 Since the observability of all other agents $i\neq j$ remains the same, we consider them as a conceptual agent denoted by $-j$.
 For simplicity, we use $\pi^*$ to represent ${\pi}^*_{(j)}$, and $\pi$ to represent ${\pi}*_{(j)}$ in distribution functions.
Similar to the illustrative example, we define the action value vector corresponding observation $o_j$, i.e.,
\begin{equation} \label{equ:vector_re_write_Q_1_app}
    \begin{aligned}
    \bar{Q}^*(o_j)=[\widetilde{Q}^*(o_{-j},o_j), \forall{o_{-j}}],
    \end{aligned}
\end{equation}
where $o_{-j}$ are the observations of all other agents and $\widetilde{Q}^*(o_{-j},o_j)$ is a vector of action values weighted by marginalized visitation probabilities $d_{\mu}^{\pi}(o_{-j}|o_j)$ and corresponding to different actions:
\begin{equation} \label{equ:vector_re_write_Q_1_0_app}
   \begin{aligned}
    \widetilde{Q}^*(o_{-j},o_j)=[Q^*(o_{-j},o_j,\boldsymbol{a})\cdot d_{\mu}^{\pi}(o_{-j}|o_j), \forall \boldsymbol{a}], \forall{o_{-j}} \sim d_{\mu}^{\pi}(o_{-j}|o_j),\forall{\boldsymbol{a}} \sim \mu^{\pi}(o_{-j},o_j,\boldsymbol{a}).
    \end{aligned}
\end{equation}
Here $\mu^{\pi}(\boldsymbol{o},\boldsymbol{a})$ denotes the observation-action distribution given a policy $\pi$, which means $\mu^{\pi}(\boldsymbol{o},\boldsymbol{a}) = d_{\mu}^{\pi}(\boldsymbol{o})\pi(\boldsymbol{a}|\boldsymbol{o})$. 
We also denote $d_{\mu}^{\pi}(m)$ is the visitation probability of message $m$, and $\bar{d}_m(o_j)$ is the marginalized probability of $o_j$ in cluster $m$, i.e.,
\begin{equation} \label{eq:d_m_app}
    d_{\mu}^{\pi}(m)=\sum_{o_j \sim m^{}}d_{\mu}^{\pi}(o_j),\ \ 
    \bar{d}_m(o_j)=\frac{d_{\mu}^{\pi}(o_j)}{d_{\mu}^{\pi}(m)},
\end{equation}

Since the optimal return $J(\pi^*_{(j)})$ is calculated from the action values by $\max_{\boldsymbol{a}} Q(\boldsymbol{o}, \boldsymbol{a})$, we rewrite this in the vector form by defining a maximization function $\Phi_{\rm max}(\widetilde{Q}^*(o_{-j},o_j))$ that returns the largest component of vector $\widetilde{Q}^*(o_{-j},o_j)$. With slight abuse of notations, we also define $\Phi_{\rm max}(\bar{Q}^*(o_j))=\sum_{o_{-j}} \Phi_{\rm max}(\widetilde{Q}^*(o_{-j},o_j))$ as the expected average return conditioned on $o_j$. Then $\sum_{o_j \sim m}\bar{d}_m(o_j) \cdot \Phi_{\rm max}(\bar{Q}^*(o_j))$ could be defined as selecting the action from optimal policy $\pi^*_{(j)}$ where agent chooses different action distribution to maximize $(\bar{Q}^*(o_j))$. 
We could re-write this term with the action-value function as:
\begin{equation} \label{equ:Phi_2_app}
    \setlength{\abovedisplayskip}{1pt}
    \setlength{\belowdisplayskip}{1pt}
    \begin{aligned}
    &\sum_{o_j \sim m}\bar{d}_m(o_j) \cdot \Phi_{max}(\bar{Q}^*(o_j))\\
    &=\sum_{o_j \sim m}\bar{d}_m(o_j) \cdot \sum_{o_{-j}} max_{\boldsymbol{a}}[Q^*(o_{-j},o_j,\boldsymbol{a})\cdot d_{\mu}^{\pi}(o_{-j}|o_j)],\\
    &=\sum_{o_j \sim m}(\frac{d_{\mu}^{\pi}(o_j)}{d_{\mu}^{\pi}(m)}) \cdot \sum_{o_{-j}} max_{\boldsymbol{a}}[Q^*(o_{-j},o_j,\boldsymbol{a})\cdot d_{\mu}^{\pi}(o_{-j}|o_j)],\\
    &=(\frac{1}{d_{\mu}^{\pi}(m)})  \sum_{o_j \sim m} d_{\mu}^{\pi}(o_j)  \sum_{o_{-j}} max_{\boldsymbol{a}}[Q^*(o_{-j},o_j,\boldsymbol{a}) d_{\mu}^{\pi}(o_{-j}|o_j)],
    \end{aligned}
\end{equation}
then we used the fact that while $J(\pi^*_{(j)})$ conditioning on complete $o_j$ can achieve maximum for each vector $\bar{Q}^*(o_j)$, policy $\pi_{(j)}$ is conditioned on messages $m_{ij}$ rather than complete $o_j$ and thus must take the same actions for all $o_j$ in the same cluster. Hence we can construct a (potentially sub-optimal) policy to achieve $\Phi_{\rm max}(\sum_{o_j\sim m}\bar{d}_m(o_j)\cdot\bar{Q}^*(o_j))$ which provides a lower bound on $J(\pi_{(j)}, g)$. Plugging in Eq.(\ref{eq:d_m_app}), $\Phi_{\rm max}(\sum_{o_j \sim m}\bar{d}_m(o_j)\cdot\bar{Q}^*(o_j))$ can be re-written as:

\begin{equation} \label{equ:Phi_1_app}
    \setlength{\abovedisplayskip}{1pt}
    \setlength{\belowdisplayskip}{1pt}
    \begin{aligned}
    &\Phi_{\rm max}(\sum_{o_j \sim m}\bar{d}_m(o_j)\cdot\bar{Q}^*(o_j))\\
    &=\sum_{o_{-j}} max_{\boldsymbol{a}}[\sum_{o_j \sim m}\bar{d}_m(o_j)Q^*(o_{-j},o_j,\boldsymbol{a})\cdot d_{\mu}^{\pi}(o_{-j}|o_j)],\\
    &=\sum_{o_{-j}} max_{\boldsymbol{a}}[\sum_{o_j \sim m}(\frac{d_{\mu}^{\pi}(o_j)}{d_{\mu}^{\pi}(m)})Q^*(o_{-j},o_j,\boldsymbol{a})\cdot d_{\mu}^{\pi}(o_{-j}|o_j)],
    \end{aligned}
\end{equation}

Multiplying the two equations above Eq.(\ref{equ:Phi_1_app}) and Eq.(\ref{equ:Phi_2_app}) with $d_{\mu}^{\pi}(m)$ which is the visitation probability of message $m$ to replace the term $\sum_{o_j \sim m} \sum_{o_{-j}}[Q^*(o_{-j},o_j,\boldsymbol{a}^*) - Q(o_{-j},o_j,\boldsymbol{a})]\cdot d_{\mu}^{\pi}(o_{-j},o_j)$ in the Eq.(\ref{equ:regret_1_app}), we obtain an upper bound on the return gap:
\begin{equation} \label{equ:regret_2_app}
    \setlength{\abovedisplayskip}{1pt}
    \setlength{\belowdisplayskip}{1pt}
    \begin{aligned}
    &J({\pi}^*_{(j)}) - J({\pi_{(j)}}, g) \\
    &\le \sum_{m} \sum_{o_j \sim m} \sum_{o_{-j}}[Q^*(o_{-j},o_j,\boldsymbol{a}^*) - Q(o_{-j},o_j,\boldsymbol{a})]\cdot d_{\mu}^{\pi}(o_{-j},o_j),\\
    &=\sum_{m}d_{\mu}^{\pi}(m)[ \sum_{o_j \sim m}\bar{d}_m(o_j) \cdot \Phi_{\rm max}(\bar{Q}^*(o_j)) \\
    &- \Phi_{\rm max}(\sum_{o_j \sim m}\bar{d}_m(o_j)\cdot\bar{Q}^*(o_j)) ].
    \end{aligned}
\end{equation}



To quantify the resulting return gap, we denote the center of a cluster of vectors $\bar{Q}^*(o_j)$ for $o_j\sim m$ under message $m$ as:
\begin{equation} \label{eq:center_H_app}
    \bar{H}(m)=\sum_{o_j \sim m^{}}\bar{d}_m(o_j)\cdot\bar{Q}^*(o_j).
\end{equation}

We contracts clustering labels to make corresponding joint action-values $\{\bar{Q}^*(o_j): g(o_j)=m\}$ close to its center $\bar{H}(m)$. Specifically, the average cosine-distance is bounded by a small $\epsilon$ for each message $m$, i.e.,
\begin{equation} \label{equ:epsilon_app}
    \setlength{\abovedisplayskip}{1pt}
    \setlength{\belowdisplayskip}{1pt}
    \begin{aligned}
    &\sum_{o_j \sim m} D(\bar{Q}^*(o_j),\bar{H}(m))\cdot d_{\mu}^{\pi}(m) \\
    & \le \sum_{o_j \sim m} \epsilon(o_j) \cdot d_{\mu}^{\pi}(m) \le \epsilon, \ \ \ \forall{m} \in M.
    \end{aligned}
\end{equation}
where $\epsilon(o_j)= \sum_{o_j \sim m} D_{cos}( \bar{Q}^*(o_j), \bar{H}(m))/K$, $K$ is the number observations with the same label $m$, and $D_{cos}(A,B)=1-\frac{A\cdot B}{||A||||B||}$ is defined as the 1 abstract the cosine similarity $cos(\theta)=\frac{A\cdot B}{||A||||B||}$ between two vectors $A$ and $B$. 

For each pair of two vectors $\bar{Q}^*(o_j)$ and $\bar{H}(m)$ with $D(\bar{Q}^*(o_j),\bar{H}(m)) \le \epsilon(o_j)$, we use $\cos{\theta_{o_j}}$ to denote the cosine-similarity between each $\bar{Q}^*(o_j)$ and its center $\bar{H}(m)$. Then we have the cosine distance $D(\bar{Q}^*(o_j),\bar{H}(m)) =1- \cos{\theta_{o_j}} \le \epsilon(o_j)$.
By projecting $\bar{Q}^*(o_j)$ toward $\bar{H}(m)$, $\bar{Q}^*(o_j)$ could be re-written as $\bar{Q}^*(o_j) = Q^{\perp}(o_j) + \cos{\theta_{o_j}} \cdot \bar{H}_m$,
where $Q^{\perp}(o_j)$ is the auxiliary vector orthogonal to vector $\bar{H}_m$.

Then we plug the decomposed Q vectors $\bar{Q}^*(o_j)$ into the return gap we got in Eq.(\ref{equ:regret_2_app}), the first part of the last step of Eq.(\ref{equ:regret_2_app}) is bounded by:
\begin{equation} \label{equ:d_phi_1_app}
    \setlength{\abovedisplayskip}{1pt}
    \setlength{\belowdisplayskip}{1pt}
    \begin{aligned}
    &\sum_{o_j \sim m}\bar{d}_m(o_j) \cdot \Phi_{max}(\bar{Q}^*(o_j))\\
    &=\sum_{o_j \sim m}\bar{d}_m(o_j) \cdot  \Phi_{max}\left(Q^{\perp}(o_j) + \bar{H}(m) \cdot \cos{\theta_{o_j}}\right),\\
    & \le \sum_{o_j \sim m}\bar{d}_m(o_j) \cdot [\Phi_{max}(Q^{\perp}(o_j)) + \Phi_{max}(\bar{H}(m) \cdot \cos{\theta_{o_j}})],\\
    & \le \sum_{o_j \sim m}\bar{d}_m(o_j) \cdot [\Phi_{max}(Q^{\perp}(o_j)) + \Phi_{max}(\bar{H}(m)) ].
    \end{aligned}
\end{equation}

We use $||\alpha||_2$ to denote the L-2 norm of a vector $\alpha$. Since the maximum function $\Phi_{\rm max}(Q^{\perp}(o_j))$ can be bounded by the $L_2$ norm $C \cdot ||Q^{\perp}(o_j)||_2$ for some constant $C$. We define a constant $Q_{\rm max}$ as the maximine absolute value of $\bar{Q}^*(o_j)$ in each cluster as $Q_{\rm max}=max_{o_j}||\bar{Q}^*(o_j)||_2$.  Since $Q^{\perp}(o_j)=\bar{Q}^*(o_j)\cdot sin(\theta)$, and $|sin(\theta)| = \sqrt{1-cos^2(\theta)}=\sqrt{1-[1-\epsilon(o_j)]^2}$, the maximum value of $Q^{\perp}(o_j)$ could also be bounded by $Q_{\rm max}$, i.e.:
\begin{equation}\label{eq:q_max_app}
\setlength{\abovedisplayskip}{1pt}
\setlength{\belowdisplayskip}{1pt}
    \begin{aligned}
    &\Phi_{max}(Q^{\perp}(o_j)) \le C\cdot||Q^{\perp}(o_j)||_2,\\
    & \le C\cdot ||\bar{Q}^*(o_j)||_2 \cdot |\sin{\theta}|,\\
    &= C\cdot Q_{\rm max} \cdot \sqrt{1-[1-\epsilon(o_j)]^2},\\
    & \le O(\sqrt{\epsilon(o_j)}Q_{\rm max}) .
    \end{aligned}
\end{equation}

Plugging Eq.(\ref{eq:q_max_app})into Eq.(\ref{equ:d_phi_1_app}), we have:
\begin{equation} \label{equ:d_phi_2_app}
    \setlength{\abovedisplayskip}{1pt}
    \setlength{\belowdisplayskip}{1pt}
    \begin{aligned}
    &\sum_{o_j \sim m}\bar{d}_m(o_j) \cdot \Phi_{max}(\bar{Q}^*(o_j))\\
    & \le \sum_{o_j \sim m}\bar{d}_m(o_j) \cdot [\Phi_{max}(Q^{\perp}(o_j)) + \Phi_{max}(\bar{H}(m)) ],\\
    & \le \sum_{o_j \sim m}\bar{d}_m(o_j) \cdot [O(\sqrt{\epsilon(o_j)}Q_{\rm max})+ \Phi_{max}(\bar{H}(m)) ],\\
    &= \sum_{o_j \sim m}\bar{d}_m(o_j) \cdot [O(\sqrt{\epsilon(o_j)}Q_{\rm max})+ \Phi_{max}(\bar{H}(m)) ],
    \end{aligned}
\end{equation}

It is easy to see from the last step in Eq.(\ref{equ:regret_2_app}) that the return gap $J({\pi}^*_{(j)}) - J({\pi_{(j)}}, g)$ is bounded by the first part $\sum_{m}d_{\mu}^{\pi}(m)[ \sum_{o_j \sim m}\bar{d}_m(o_j) \cdot \Phi_{\rm max}(\bar{Q}^*(o_j))$ which has an upper bound derived in Eq.(\ref{equ:d_phi_2_app}), then the return gap could also be bounded by the upper bound in  Eq.(\ref{equ:d_phi_2_app}), i.e.,
\begin{equation} \label{equ:regret_3_app}
    \setlength{\abovedisplayskip}{1pt}
    \setlength{\belowdisplayskip}{1pt}
    \begin{aligned}
    &J({\pi}^*_{(j)}) - J({\pi_{(j)}}, g) \\
    &=\sum_{m}d_{\mu}^{\pi}(m)[ \sum_{o_j \sim m}\bar{d}_m(o_j) \cdot \Phi_{max}(\bar{Q}^*(o_j)) - \Phi_{max}(\sum_{o_j \sim m}\bar{d}_m(o_j)\cdot\bar{Q}^*(o_j)) ],\\
    & = \sum_{m}d_{\mu}^{\pi}(m)\left[\sum_{o_j \sim m}\bar{d}_m(o_j) \cdot [O(\sqrt{\epsilon(o_j)}Q_{\rm max})+ \Phi_{max}(\bar{H}(m))]\right] - \sum_{m}d_{\mu}^{\pi}(m)\left[\Phi_{max}(\sum_{o_j \sim m}\bar{d}_m(o_j)\cdot\bar{Q}^*(o_j)) \right],\\
&\le \sum_{m} d_{\mu}^{\pi}(m) \cdot [\sum_{o_j \sim m}\bar{d}_m(o_j) \cdot O(\sqrt{\epsilon(o_j)}Q_{\rm max})] ,\\
&\le \sum_{m} d_{\mu}^{\pi}(m)  \cdot O\left(\sqrt{\sum_{o_j \sim m} \bar{d}_m(o_j)\cdot \epsilon(o_j)}Q_{\rm max}\right)  , \\
    &\le \sum_{m}   d_{\mu}^{\pi}(m) \cdot O(\sqrt{\epsilon}Q_{\rm max}) ,\\  
    &=   O(\sqrt{\epsilon}Q_{\rm max}).
    \end{aligned}
\end{equation}

we can derive the desired upper bound $J(\pi^*_{(j)})-J(\pi_{(j)}, g) \le  O(\sqrt{\epsilon}Q_{\rm max})$ for the two policies in Lemma 4.2 in the main paper.
\end{proof}

\subsection{Proof of Theorem 4.3:}
We construct a sequence of policies beginning from $\pi^{*}$, end with $\pi$, which has a process of changing by replacing each observation $o_j$ by label $m_j$, one at a time: $\pi^{*}=[\pi_i{(a_i|o_1,o_2,\dots,o_n)}, \forall i$] $\to$ $[\pi_i{(a_i|m_1,o_2,\dots,o_n)}, \forall i$] $\to$ $[\pi_i{(a_i|m_1,m_2,\dots,o_n)},\forall i]$ $\to$ $\dots$ $\to$ $\pi$=$[\pi_i({a_i|m_1,m_2,\dots,m_n}),\forall i]$. Applying the results from Lemma 4.2 for $n$ times, we can quantify the desired upper bound between $J({\pi}^*)$ and $J({\pi},g)$ in this theorem.

\subsection{Detailed process of calculation of the average cosine distance defined in Eq.(6) in the main paper}
\begin{itemize}
    \item Within each cluster, we compute the cosine distance between every action-value vector and the  clustering center, thereby obtaining the value $\epsilon(o_j)$.
    
    \item To obtain the expectation of the cosine distance for each particular cluster denoted by $m$, we compute the product of each $\epsilon(o_j)$ with $\bar{d}_m(o_j)$, where $\bar{d}_m(o_j)$ represents the marginalized probability of $o_j$ in cluster $m$. 
    \item Having computed the expected cosine distance for each cluster $m$, we proceed to obtain the average cosine distance across all clusters. This is accomplished by multiplying $d_{\mu}^{\pi}(m)$ to each $\sum_{o_j\sim m} \bar{d}_m(o_j) \cdot \epsilon(o_j)$, where $d_{\mu}^{\pi}(m)$ denotes the visitation probability of message $m$.
\end{itemize}

\subsection{Detailed Explanations for the Illustrative example}
We want to use the example to illustrate the key intuitions behind our theoretical derivations, i.e., why the return gap in MARL with communications can be upper bounded by a term involving the average cosine distance $\epsilon$. 

The left figure in Fig.1 shows a centralized action-value(Q) table of a simple two-agent matrix game where the transition probability of the observations does not depend on policies and all observations are equally likely for each agent. The observation spaces are $\Omega_1 = \{o_{11},o_{12}\}$ and $\Omega_2 = \{o_{21},o_{22},o_{23},o_{24}\}$. Both agents have discrete action space, where $\mathcal{A}_1 = \{a_{11},a_{12}\}$ and agent $2$ has a deterministic policy to choose its actions based on its observation. $\bar{Q}^*(o_{2i})$ is the centralized action-value function vectors for each fixed $o_{2i}$, for example, the first column of Q vectors is $\bar{Q}^*(o_{21})=[53.2, 42.9, 31.8, 22.9]$. The goal of the two agents is to maximize the average return which could be achieved by trying to take actions to maximize the centralized Q values.

In a partially observable environment where agents can only see their own observations, if agent $1$ observes $o_{11}$, the average Q values of taking action $a_{11}$ is $(53.2+4.5+58.5+0.3)/4=29.125$, which is smaller than $(42.9+1.2+64.0+16.1)/4=31.05$ by taking action $a_{12}$, so agent $1$ will take action $a_{12}$ and get the average return $31.05$. 
Same calculations for observing $o_{12}$, then the agent $1$ will get the total average return for two possible observations as $(31.05 + 35.6)/2 = 33.325$ with partial observation. 
In this situation, the agent $1$ tries to choose the actions maximizing the Q values based on the average reward information, but since agent $1$ does not know the observation of agent $2$, it could not always choose the best action and result in some loss of average returns.

In a fully observable environment where agents can see others' observations, agent $1$ could choose actions to maximize the Q values (shown as the bold numbers) which could be located by observations of both agents. The agent $1$ will get the largest average reward as $(53.2+4.5+64.0+16.1)/4=34.425$ when observing $o_{11}$, and get $(31.8+28.0+34.1+81.5)/4=43.85$ when observing $o_{12}$, then the total average returns is $(34.425 + 43.85)/2=39.1375$.

The idea of our communication algorithm is to cluster the four observations of agent $2$ which are invisible to agent $1$ into $2^n$ clusters and send the cluster labels as a $n$-bit communication message to agent $1$. Here we start with $n=1$, i.e., only use $1-$bit message in this example. The most important message agent $2$ wants to send is where is the largest Q value located based on its own observation. Therefore, we use cosine distance $1-\cos(\theta)$ as the clustering metric which is defined as the 1 abstract the cosine of angle $\theta$ between two vectors and two vectors with similar directions(smaller angle $\theta$) have a smaller cosine distance. In this way, this clustering method could put the Q vectors with similar structures into the same cluster. We define the $\bar{Q}^*(o_{2i})$ as the centralized action-value function vectors for each fixed $o_{2i}$, for example, the Q vectors of fixed observation $o_{21}$ of agent $2$ is $\bar{Q}^*(o_{21})=[53.2, 42.9, 31.8, 22.9]$. 

Then we calculate the cosine distance between all possible pairs of Q vectors of agent $2$ -- $\bar{Q}^*(o_{2i})$ and $\bar{Q}^*(o_{2j})$, $\forall i,j \in \{1,2,3,4\}$ and list the cosine distance between each pair in the right figure in Fig.1. Based on the cosine distance values, we put vectors with smaller cosine distance between them into the same cluster. For example, as for the observations of agent $2$ which are unknown to agent $1$, we put $o_{21},o_{23}$ into the same cluster and $o_{22},o_{24}$ into another cluster and use clustering labels $0,1$ respectively as communication messages being sent to agent $1$. 
When agent $1$ observes $o_{11}$ and receive communication messages $0$ from agent $2$, it will know that agent $2$ observes $o_{21}$ or $o_{23}$, it will take action $a_{11}$ not taking the $a_{12}$ as it acted with only local observation, this is because $53.2+58.5=117.5$ is larger than $42.9+64.0=106.9$. Same as observing $o_{11}$ and getting label $1$, agent $1$ will take action $a_{12}$ since $1.2+16.1=17.3$ is larger than $4.5+0.3=4.8$, then the average return when agent $1$ observing $o_{11}$ becomes $(117.5+17.3)/4=33.7$. 
We do the same calculations when agent $1$ observes $o_{12}$ and get the average return as $43.85$, then the total average returns for all conditions of agent $1$ is $(33.7+43.85)/2=38.775$. 

Therefore, the regret between policies with our method and the full observation is $39.1375-38.775=0.3625$, and the regret between policies with partial observation and full observation is $39.1375-33.325=5.8125$. 
We could see that using clustering labels as communication messages leads to a better decision with a small regret. And this clustering method not only works with discrete observation space but also could be compatible with any online clustering algorithm in order to use it with a continuous state space environment. In addition, the centralized Q table is available using centralized training decentralized execution algorithms mentioned in Section 2.

\section{Return-Gap-Minimization Communication (RGMComm) Algorithm Design}

\subsection{Detailed Explanation of the communication loss function $\mathcal{L}(g_{\xi_i})$}
The message generation functions $g=\{g_1,\dots,g_n\}$ of all agents are approximated using DNNs parameterized by $\xi=\{\xi_1,\dots,\xi_n\}$. For each agent $j$, we first sample a random minibatch of $K_1$ samples $\mathcal{X}_{j}=(\boldsymbol{o}^{k_1},\boldsymbol{a}^{k_1},R^{k_1},\boldsymbol{o}'^k)$ from the transitions recorded in replay buffer $\mathcal{R}$ which contains the observation-action pairs from all agents including agent $j$. Then we sample a set $\mathcal{X}_{-j}=(\boldsymbol{o}_{-j}^{k_2},\boldsymbol{a}_{-j}^{k_2})$ from $\mathcal{X}_{j}$ which are top $K_2$ frequent observation-action pairs in sample $\mathcal{X}_{j}$ after removing $o_j$ and $a_j$ from $\mathcal{X}_{j}$. 
Then we form the sampled trajectories by combining $(o_j,a_j)$ in $\mathcal{X}_{j}$ and $(\boldsymbol{o}_{-j},\boldsymbol{a}_{-j})$ in $\mathcal{X}_{-j}$ as $\mathcal{D}=(\boldsymbol{o}^{k_1k_2},\boldsymbol{a}^{k_1k_2},R^{k_1k_2},\boldsymbol{o}'^{k_1k_2})$. 
To obtain the action-values for clustering, we query the critic networks with $\mathcal{D}$ as the input to get the $\hat{Q}_{\omega}(o_j,\boldsymbol{o}_{-j},a_j, \boldsymbol{a}_{-j})$, which approximates action-value vectors $\bar{Q}^*(o_j)$ defined in Eq (\ref{equ:vector_re_write_Q_1}) and illustrative example in Sec.~\ref{sec:theory}. We use $\bar{Q}^*(o_j)$ instead of $\hat{Q}_{\omega}$ in the following part to make it consistent with the theoretical results. 
The message $m_j=g_{\xi_j}(o_j)$ is updated by minimizing a Regularized Information Maximization(RIM) loss function~\cite{imsat} $\mathcal{L}(g_{\xi_j})$ in terms of $\bar{Q}^*(o_j)$:
\begin{equation} \label{eq: commloss_app}
    \setlength{\abovedisplayskip}{1pt}
    \setlength{\belowdisplayskip}{1pt}
    \begin{aligned}
&\mathcal{L}(g_{\xi_i}) = L_{CD} - \lambda L_{MI}, \\
&L_{CD}= \sum_{p=1}^{K_1}\sum_{q \in N_{K_3}(p)} \left[D_{cos}(\bar{Q}^*(o_j^p),\bar{Q}^*(o_j^q)\right]\|m_{j}^p-m_{j}^q\|^2, \\
&L_{MI} = I(o_{j};m_{j})=H(m_j)-H(m_j|o_j),
    \end{aligned}
\end{equation}
where $L_{CD}$ is a clustering loss in the form of  Locality-preserving loss~\cite{localitypreservingloss}, it preserves the locality of the clusters by pushing nearby data points of action-value vectors together. Inside $L_{CD}$, $o_j^p \in \mathcal{X}_j, p=1,\dots,K_1$ is the sampled observation $o_j$, $N_{K_3}(p)$ is the set of $K_3$ nearest neighbors of $\bar{Q}^*(o_i^p)$, we use $D_{cos}$ as the metric to define the neighbors, where $D_{cos}$ is the cosine-distance between the points $\bar{Q}^*(o_j^p)$ and its neighbor $\bar{Q}^*(o_j^q)$. And the mutual information loss, $L_{MI}$, is the summation of mutual information between observation $o_{j}$ and message $m_{j}$. It measures how much information is shared between the observation and message and is zero if $o_{j}$ and $m_{j}$ are independent. Here we measure the mutual information loss as the difference between marginal entropy and conditional entropy, $H(m_j)-H(m_j|o_j)$, which are calculated as:
\begin{equation} \label{eq: L_MI_app}
    \setlength{\abovedisplayskip}{1pt}
    \setlength{\belowdisplayskip}{1pt}
    \begin{aligned}
&H(m_j)=h(p_{\xi_j}(m_j))=h(1/K_1)\sum_{p=1}^{K_1}p_{\xi_j}(m_j|o_j),\\
&H(m_j|o_j)=1/K_1)\sum_{p=1}^{K_1}h(p_{\xi_j}(m_j|o_j).
    \end{aligned}
\end{equation}
where $h(p(x))=-\sum_{x'} p(x')\log p(x')$ is the entropy function. \textbf{Increasing} the marginal entropy $H(m_j)$ that measures the randomness of the $m_j$ promotes more diversity in the clusters and leads to roughly equal-sized clusters (uniform cluster sizes). And the conditional entropy $H(m_j|o_j)$ measures the unpredictability or randomness of $m_j$ given $o_j$, by \textbf{decreasing} this entropy, the goal is to make the assignment of observations to clusters more certain and less ambiguous, In other words, each observation should clearly belong to one cluster without much overlap or confusion with other clusters. The two objectives ensure that while clusters are of roughly the same size, the assignments of observations to these clusters are clear and definitive. 

In summary, minimizing $L_{CD}$ ensures we are minimizing the cosine-distance to bound the return gap between the full-observability policy and partial-observability policy with communication, and maximizing $L_{MI}$ ensures the clusters have uniform sizes and cluster assignments for observations are unambiguous. A hyper-parameter $\lambda \in \mathbf{R}$ trades off the two terms.
Before calculating the cosine-distance, action-value vectors are normalized and went through the activation function $f((\bar{Q}^*(o_i^p)-\alpha)/\beta)$ for a better clustering result. The choice of activation function is explored as an ablation study in the experiments section.

\subsection{RGMComm message generation training process Pseudo-Code}

The message generation training process of our proposed Return-Gap-Minimization Communication (RGMComm) algorithm is summarized here in Algorithm~\ref{alg:comm_main}.
\begin{algorithm}[htbp]
\caption{Training message generation}\label{alg:comm_main}
{\begin{algorithmic}[1]
\State {\bfseries Input:}$K_1$, $K_2$, $K_3$, $\lambda$, Replay buffer $\mathcal{R}$, current parameters $\omega$, $\xi=\{\xi_1,\dots,\xi_n\}$.
   \For{$t=1$ {\bfseries to} $T$}
   \For{agent $j$ to n}
   \State Get top-$K_1$ samples $\mathcal{X}_{j}=(\boldsymbol{o}^{k_1},\boldsymbol{a}^{k_1},R^{k_1},\boldsymbol{o}'^{k_1})$ from reply buffer $\mathcal{R}$;
   \State Get top-$K_2$ samples $\mathcal{X}_{-j}=(\boldsymbol{o}_{-j}^{k_2},\boldsymbol{a}_{-j}^{k_2})$ from $\mathcal{X}_{j}$ from $\mathcal{X}_{j}$;
   \State combining $(o_j,a_j)$ in $\mathcal{X}_{j}$ and $(\boldsymbol{o}_{-j},\boldsymbol{a}_{-j})$ in $\mathcal{X}_{-j}$ as $\mathcal{D}=(\boldsymbol{o}^{k_1k_2},\boldsymbol{a}^{k_1k_2},R^{k_1k_2},\boldsymbol{o}'^{k_1k_2})$;
   \State Query the critic $\omega$ with $\mathcal{D}$ as the input to get the $\hat{Q}_{\omega}(o_j,\boldsymbol{o}_{-j},a_j, \boldsymbol{a}_{-j})$;
   \State Update $g_{\xi_{j}}$ by minimizing the loss $L(g_{\xi_{j}})$ defined in the main paper;
   \EndFor

   \EndFor
\State {\bfseries Output:} Message functions:  $m_j = g_{\xi_j}(o_j)$.
\end{algorithmic}}
\end{algorithm}

\subsection{RGMComm - Complete Training Pseudo-code}
The complete training procedure of RGMComm algorithm is summarized in Algorithm~\ref{algo:RGMComm_training} below.
\begin{algorithm}[htbp]
	\caption{\rm RGMComm training Algorithm} \label{algo:RGMComm_training}
 \small 
	{\begin{algorithmic}[1]
	    \State \textbf{Input}: sampling size $K_1$, $K_2$, nearest neighbors' size $K_3$, reply buffer size $R$, exploration constant $\tau_1$, $\tau_2$, $\tau_3$, learning rate $\alpha_0,\beta_0,\eta_0$;
		\State Initialize network weights $(\theta ,\omega, \xi )$ for each agent $j$ at random;
		\State Initialize target weights $(\theta ', \omega ',\xi') \gets (\theta , \omega,\xi )$ for each agent $j$ at random;
		\For {episode$=1,2,\ldots,\textbf{max-episode}$}
		\State Start a new episode;
		\State Initialize a random process $\mathcal{N}$ for action exploration;
            \State Receive initial observation $\boldsymbol{o}=[o_1,\dots,o_n]$;
		\For {$t=1$ to \textbf{max-episode-length} } 
            \State for each agent $j$, receive initial observation $o_j$, receive initial labels $m_{-j}=\{g_{\xi_i}(o_i),\forall i \neq j\}$; 
            \State Generating communication labels $m_j=g_{\xi_j}(o_j)$;
            \State Select action $a_j =\pi_{\theta_j }(o_j,m_{-j})+\mathcal{N}_t$ w.r.t. the current policy and exploration;
            \State Execute actions $\boldsymbol{a}=[a_1,\dots,a_n] $, observe reward $R$, new observation $\boldsymbol{o}'$;
            \State Store $(\boldsymbol{o} ,\boldsymbol{a} ,R ,\boldsymbol{o}' )$ in replay buffer $\mathcal{R}$;
		\For{agent $j=1$ to $n$}
		\State Sample top-$K_1$ samples $\mathcal{X}_{j}=(\boldsymbol{o}^{k_1},\boldsymbol{a}^{k_1},R^{k_1},\boldsymbol{o}'^{k_1})$ from reply buffer $\mathcal{R}$;
            \State Sample top-$K_2$ samples $\mathcal{X}_{-j}=(\boldsymbol{o}_{-j}^{k_2},\boldsymbol{a}_{-j}^{k_2})$ from $\mathcal{X}_{j}$ from $\mathcal{X}_{j}$;
            \State Combining $(o_j,a_j)$ in $\mathcal{X}_{j}$ and $(\boldsymbol{o}_{-j},\boldsymbol{a}_{-j})$ in $\mathcal{X}_{-j}$ as $\mathcal{D}=(\boldsymbol{o}^{k_1k_2},\boldsymbol{a}^{k_1k_2},R^{k_1k_2},\boldsymbol{o}'^{k_1k_2})$;
            \State Query Centralized Critic parameterized by $\omega$, and get $K_1$ vectors $\hat{Q}_{\omega}(o_j^p,\boldsymbol{o}_{-j}^p,a_j^p,\boldsymbol{a}_{-j}^p)$, each has a length of $K_2$;
            \For{$p=1$ {\bfseries to} $K_1$}
            \State Normalize $\hat{Q}_{\omega}(o_j^p,\boldsymbol{o}_{-j}^p,a_j^p,\boldsymbol{a}_{-j}^p))$ and go through activation function $f$;
            \EndFor
            \State Getting the set of $K_3$ nearest neighbors of each $\hat{Q}_{\omega}(o_j^p,\boldsymbol{o}_{-j}^p,a_j^p,\boldsymbol{a}_{-j}^p), \forall p$;
            \State Saving $K_3$ nearest neighbors in set $N_{K_3}(p)=\{\hat{Q}_{\omega}(o_j^q,\boldsymbol{o}_{-j}^q,a_j^q,\boldsymbol{a}_{-j}^q)), q=1,\dots,K_3\}$
            \State Update $g_{\xi_j}$ by minimizing the loss $L(g_{\xi_j})=L_{CD} - \lambda L_{MI}$:
            \begin{equation*}
                \mathcal{L}(g_{\xi_j}) = \sum_{p=1}^{K_1}\sum_{q \in N_{K_3}(p)} \left[D_{cos}(\hat{Q}_{\omega}(o_j^p,\boldsymbol{o}_{-j}^p,a_j^p,\boldsymbol{a}_{-j}^p),\hat{Q}_{\omega}(o_j^q,\boldsymbol{o}_{-j}^q,a_j^q,\boldsymbol{a}_{-j}^q)\right]\|g_{\xi_j}(o_j^p)-g_{\xi_j}(o_j^q)\|^2 -\lambda \sum_{p=1}^{K_1} I(o_{j}^p;g_{\xi_j}(o_j^p));
            \end{equation*}
            \State Set $y^p=R^p + \gamma \hat{Q}_{\omega'}(\boldsymbol{o}'^p,\boldsymbol{a})|_{\boldsymbol{a}'=\pi'(\boldsymbol{o}'^p)}$;
            \State Update critic by minimizing the loss $\mathcal{L}(\omega) = \frac{1}{K_1}\sum_{p}[(\hat{Q}_{\omega}(\boldsymbol{o}^p,\boldsymbol{a}^p)-y^p)^2]$
		\State Update actor using the sampled policy gradient: 
  \[\nabla_{\theta_j}J(\theta_j) \approx \frac{1}{K_1}\sum_{p} \left[\nabla_{\theta_j} \log \pi_{\theta_j}(a_j|o_j^p,m_{-j}) \hat{Q}_{\omega}^{\pi_{\theta_j}} (\boldsymbol{o}^p,a_1^p,\dots,a_j,\dots,a_n^p)|_{a_j=\pi_{\theta_j}(o_j^p,m_{-j})}\right].\]
	    \State Update actor network parameters: 
	    $\theta  \gets \theta +\alpha_t\delta_{\theta }$; 
	    \State Update critic network parameters: 
	    $\omega  \gets \omega +\beta_t \delta_{\omega }$;
            \State Update communication network parameters: 
	    $\xi  \gets \xi +\eta_t \delta_{\xi }$;
        \If{UPDATE TARGET NETWORKS}
		\State Update actor target network parameters:
		\begin{equation*}
		 \theta' \gets \tau_1 \theta  + (1-\tau_1)\theta'  
		\end{equation*}
		\State Update critic target network parameters: 
		\begin{equation*}
		  \omega' \gets \tau_2 \omega  + (1-\tau_2)\omega'
		\end{equation*}
            \State Update communication target network parameters: 
		\begin{equation*}
		  \xi' \gets \tau_3 \xi  + (1-\tau_3)\xi'
		\end{equation*}
	    \EndIf
		\EndFor   	
		\EndFor
		\EndFor
	\end{algorithmic}}
\end{algorithm}
\clearpage

\section{Experiments}

\subsection{Code}
Codes of RGMComm are available at: 
https://anonymous.4open.science/r/RGMComm-72AC/
  
\subsection{Baseline Details: }
We consider three types of baseline communication architectures that are derived from prior work: 
\begin{enumerate}
    \item \textbf{continuous} communication: we choose SARNet~\cite{sarnet} that outperforms all other algorithms, which leverages a memory-based attention network to produce a stream of uninterrupted, continuous messages; 
    \item one-hot \textbf{discrete} communication: CommNet~\cite{commnet}, MADDPG~\cite{maddpg}, IC3Net~\cite{ic3net}, TarMAC\cite{tarmac}, and SARNet \cite{sarnet}. We enable the discrete communication mode reported in these algorithms, which communicates with discrete but unstructured vectors. The discrete one-hot baselines allow agents to output a one-hot vector within the communication space using the Gumbel softmax operation to calculate gradients and update the communication networks via backpropagation; 
    \item the most recent prototype-based method Vector-Quantized Variational Information Bottleneck (VQ-VIB)~\cite{tucker2022trading}, which trains an autoencoder and use the encoded latent representation followed by a quantization layer  that discretizes $z$ into the nearest embedding vector to communicate.

\end{enumerate}
\subsection{Environment Details}

\begin{figure}[h]
  \centering
  \begin{subfigure}[b]{0.325\textwidth}
    \centering
    \includegraphics[width=\textwidth]{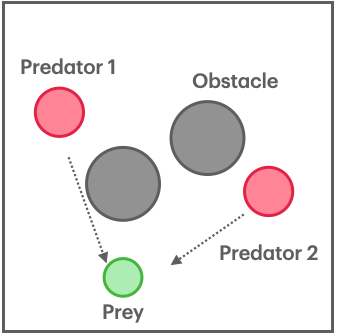}
    \caption{Predator-Prey}
    \label{fig:simple_tag}
  \end{subfigure}
  \hfill
  \begin{subfigure}[b]{0.325\textwidth}
    \centering
    \includegraphics[width=\textwidth]{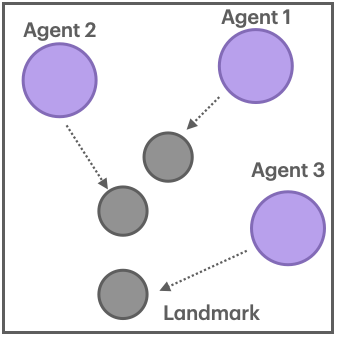}
    \caption{Cooperation Navigation}
    \label{fig:simple_spread}
  \end{subfigure}
  \begin{subfigure}[b]{0.325\textwidth}
    \centering
    \includegraphics[width=\textwidth]{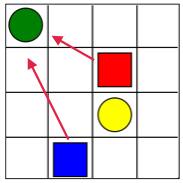}
    \caption{Maze}
    \label{fig:maze}
  \end{subfigure}
  \caption{(a) and (b): Illustrations of the multi-agent particle environment in the experiments: Predator-Prey, Cooperative Navigation respectively. (c): Maze environment for the message-interpretation study.}
\end{figure}

\textbf{Predator-Prey:} This task involves a slower-moving team of $N$ communicating predators(modeled as the learning RGMComm agents) chasing $K$ faster moving preys in an environment with $L$ obstacles. Each predator only gets its own local observation and a limited view of its nearest preys and obstacles without knowing other predators' positions and velocities. Each time a predator collides with a prey, the predator is rewarded with $+10$ while the prey is penalized with $−10$. The predators are also given a group reward $-0.1 \cdot \sum_{n=1}^{N}\min\{d(\textbf{p}_n,\textbf{q}_k), k \in \{1,\dots,K\} \}, n \in\{1,\dots,N\}$, where the $\textbf{p}_n$ and $\textbf{q}_k$ are the coordinates of predator $n$ and prey $k$. The reward of the predators will be decreased for increased distance from preys.

\textbf{Cooperative Navigation:} In this environment, $N$ agents need to cooperate to reach $L$ landmarks while avoiding collisions. Each agent predicts the actions of nearby agents based on its own observation and received information from other agents. The agents are penalized if they collide with each other, and positively rewarded based on the proximity to the nearest landmark. 

\textbf{RGMComm converges to higher return than full observation policy: }The proposed RGMComm algorithm sometimes even converges to a higher mean episode reward since the message generation function learned from the critic provides a succinct, discrete representation of the optimal action-value structure, leading to less noisy communication signals and allowing agents to discover more efficient decision-making policies conditioned on the message labels.

\textbf{Maze Task:}  This is a simple tabular-case multi-agent task. Consider a 4 by 4 grid world in Fig.(\ref{fig:maze}), where landmark 1(the green dot) with a high reward of $r_1$ is always placed at the higher left corner (1,1), while landmark 2(the yellow dot) with a low reward of $r_2$ is always placed at the third diagonal position (3,3). Two agents(the red and the blue square) are randomly placed on the map and can only see their own position without seeing the targets' positions or each other's positions. There is no collision or movement cost. The task is completed and the reward $r_1$ (or $r_2$) is achieved only if both agents occupy the same landmark 1 (or 2) before or at the end of the game. Otherwise, zero rewards are achieved. The game terminates in 3 steps or when the task is completed.

\subsection{Hyperparameters}
\label{appendix:hyperparameters}

\textbf{Sampling Parameter $K_1$, $K_2$ and $K_3$: } we use 256 for $K_1$ and $K_2$, use $16$ for $K_3$.

\textbf{Baselines: } we use the same Hyperparameters reported in SARNet paper~\cite{sarnet} for baselines CommNet\cite{commnet}, IC3Net~\cite{ic3net}, TarMAC\cite{tarmac}, and SARNet \cite{sarnet}. For MADDPG with partial observation and full observation, we use the same Hyperparameters as reported in the PyTorch repository https://github.com/starry-sky6688/MADDPG. For VQ-VIB, we use the same Hyperparameters reported in \cite{tucker2022trading}.

\textbf{RGMComm:}
In our experiments, the actor and critic neural networks consist of three hidden layer with 64 neurons and the Tanh activation function for actor, ReLU activation function for critic. For communication, the communication network has two hidden layers with 1200 units and use ReLU activation function. We use Adam optimizer with the learning rate of the actor-network as $0.0001$, while the critic network has a learning rate of $0.001$. This small difference makes the critic network learn faster than the actor-network. A slower learning rate of the actor could allow it to obtain feedback from the critic in each step, making the learning process more robust. For the soft update of target networks, we use $\tau = 0.01$. Finally, we use $\gamma = 0.95$ in the critic network to discount the reward and calculate the advantage functions. 
For the activation function $f((\bar{Q}^*(o_j^p)-\alpha)/\beta)$ using for normalizing the action-value vectors $\bar{Q}^*(o_j^p)$ before online clustering, we use Tanh function as $f(\cdot)$ and $\alpha=(max\{\bar{Q}^*(o_j^p)^1,\dots,\bar{Q}^*(o_j^p)^{K_1}\}+min\{\bar{Q}^*(o_j^p)^1,\dots,\bar{Q}^*(o_j^p)^{K_1}\})/2$, $\beta=max\{\bar{Q}^*(o_j^p)^1,\dots,\bar{Q}^*(o_j^p)^{K_2}\}$ to do the normalization, respectively.

\textbf{Notes:}
In order to make a fair comparison, even we use the baselines CommNet\cite{commnet}, IC3Net~\cite{ic3net}, TarMAC\cite{tarmac}, and SARNet \cite{sarnet} implemented from repository https://github.com/caslab-vt/SARNet, but we change the Hyperparameters to the same value if we have the same variables in RGMComm. 

\subsection{Agents Behavior and Agreements}
{\bf \noindent Do agents agree on the same clusters?} The online clustering is performed for each agent on samples of the action-values, as shown in our example in Figure~1. Other agents simply condition their policies on received messages during training. Thus, there is no need for agreement before training. 

{\bf \noindent Agents' agreements before training: }
Agents do not pre-agree on clusters before training; instead, RGMComm's algorithm learns the clustering labels dynamically during the training process. This online clustering allows each agent to develop its communication strategy, resulting in potentially varied messages across agents.

\subsection{Detailed Explanations for Experimental Results}

{\bf \noindent Clarifications on our experiments: }
(i) Figure~3 displays the mean episode reward for predators trained by RGMComm. (ii) The prey use independent policies and are considered as part of the stochastic environment. (iii) Increasing the number of agents means increasing the number of predators. (iv) In the evaluations, the prey's policy remains the same across all baselines and RGMComm. We train the prey once and save its trained model/policy for evaluation. Thus, better predator performance indeed implies the advantage of our RGMComm. 

{\bf \noindent Ablation study: }
The ablation study in Figure~6 demonstrates: (i) How different normalizations of the action value vectors $\bar{Q}^*(o_j)$ (which result in different distance measures) affect the performance of RGMComm. For instance, soft-max activation extracts the largest element of $\bar{Q}^*(o_j)$ for the distance calculation. This validates the importance of cosine-distance in our bound. (ii) The comparison with independent and centralized Q learning shows the effectiveness of communication using the RGMComm algorithm. These were conducted in the grid-world maze (details in Appendix~C), chosen for its clarity in showing the explainability result in Figure~7. 

{\bf \noindent Determine the number of labels: } The goal of our paper is to find the optimal communication strategy given any arbitrary constraint on the maximum number of allowed messages (i.e., finite-size discrete alphabet, as mentioned in the introduction and Remark~4). Thus, RGMComm finds the optimal communication for any given constraint of $|M|$ messages. 
In practice, $|M|$ can be determined based on, e.g., limited bandwidth and communication cost. In Figures~4~and~5(c), we have conducted experiments for different numbers of messages, to demonstrate the effectiveness of RGMComm under different constraints.

{\bf \noindent Higher reward than the policy with full observability: } Theoretically, the optimal policy with full observability should achieve the highest reward. However, such a policy would face a combinatorial complexity of size-$|O|^n$ observation space, where $|O|$ is the observation space size of a single agent and $n$ is the number of agents. This makes it difficult in practice to learn (an efficient DNN representation of) such a complex policy with full observability. On the other hand, RGMComm may have a return gap as characterized by Theorem~3, but learning such a simple policy could be more effective in practice.

\end{document}